\documentclass{article}
\usepackage[margin=1.2in]{geometry}
\usepackage{setspace}
\usepackage{amssymb}

\newcommand{\varprocess}{V_t}

\newcommand{\uncopt}{\eta^*_t}

\usepackage{multirow}
\usepackage{colortbl}
\usepackage{threeparttable}

\usepackage{microtype}
\usepackage{graphicx}
\usepackage{subfigure}
\usepackage{booktabs}
\usepackage{hyperref}

\usepackage[authoryear]{natbib}

\usepackage{amsmath}
\usepackage{amssymb}
\usepackage{mathtools}
\usepackage{amsthm}
\usepackage{amsfonts}

\usepackage[capitalize,noabbrev]{cleveref}

\theoremstyle{plain}
\newtheorem{theorem}{Theorem}[section]
\newtheorem{proposition}[theorem]{Proposition}
\newtheorem{lemma}[theorem]{Lemma}
\newtheorem{corollary}[theorem]{Corollary}
\theoremstyle{definition}
\newtheorem{definition}{Definition}

\newtheorem{remark}{Remark}

\usepackage{todonotes}

\usepackage[algo2e,lined,boxed,commentsnumbered,ruled]{algorithm2e}
\usepackage{setspace}
\RestyleAlgo{ruled}

\usepackage{accents}

\usepackage{enumitem}
\usepackage{nicefrac}       
\usepackage{textcomp}

\usepackage{comment}

\definecolor{mygray}{gray}{0.8}

\newcommand{\numberthis}{\addtocounter{equation}{1}\tag{\theequation}}

\allowdisplaybreaks

\newcommand{\lrset}[1]{\left\{ #1 \right\}}
\newcommand{\lrp}[1]{\left( #1 \right)}
\newcommand{\lrs}[1]{\left[ #1 \right]}
\newcommand{\abs}[1]{\left|{#1}\right|}

\newcommand{\R}{\mathbb{R}}
\newcommand{\N}{\mathbb{N}}

\newcommand{\calE}{\mathcal E}

\title{
Eventually LIL Regret: Almost Sure $\ln\ln T$ Regret for a sub-Gaussian Mixture on Unbounded Data}
\date{}
\author{Shubhada Agrawal$^1$, Aaditya Ramdas$^2$\\
$^1$Indian Institute of Science, $^2$Carnegie Mellon University\\
\texttt{shubhada@iisc.ac.in, aramdas@cmu.edu}\\
\smallskip \\
\today
}

\begin{document}
\maketitle

\begin{abstract}
We prove that a classic sub-Gaussian mixture proposed by Robbins in a stochastic setting actually satisfies a path-wise (deterministic) regret bound. For every path in a natural ``Ville event'' $\mathcal E_\alpha$, this regret till time $T$ is bounded by $\ln^2(1/\alpha)/V_T + \ln (1/\alpha) + \ln \ln V_T$ up to universal constants, where $V_T$ is a nonnegative, nondecreasing, cumulative variance process. (The bound reduces to $\ln(1/\alpha) + \ln \ln V_T$ if $V_T \geq \ln(1/\alpha)$.) If the data were stochastic, then one can show that $\mathcal E_\alpha$ has probability at least $1-\alpha$ under a wide class of distributions (eg: sub-Gaussian, symmetric, variance-bounded, etc.). In fact, we show that on the Ville event $\mathcal E_0$ of probability one, the regret on every path in $\mathcal E_0$ is eventually bounded by $\ln \ln V_T$ (up to constants). We explain how this work helps bridge the world of adversarial online learning (which usually deals with regret bounds for bounded data), with game-theoretic statistics (which can handle unbounded data, albeit using stochastic assumptions). In short, conditional regret bounds serve as a bridge between stochastic and adversarial betting.
\end{abstract}

\section{Introduction}
Since at least the work of Robbins~\citep{robbins1968iterated,darling1967confidence, darling1967iterated}, \emph{mixture} supermartingales have been used in conjunction with Ville’s supermartingale inequality \citep{ville1939etude} in order to deliver time-uniform concentration inequalities and confidence sequences. Since then, extensions of Robbins' method-of-mixtures have delivered tight bounds in both parametric and nonparametric settings, enabling several modern applications in sequential testing, estimation and bandits \citep{kaufmann_mm_21,howard2021time,waudby2024estimating}.  All of these works yield coverage guarantees that are inherently stochastic, and do not appear to deliver any adversarial guarantees.


In contrast, the current paper proves a \emph{deterministic, pathwise} wealth/regret inequality (valid for every data sequence) for a particular sub-Gaussian mixture method proposed by Robbins. 
What makes this regret bound interesting is that it allows the underlying sequences to have unbounded domain. Due to the unbounded domain, the worst case regret bound is obviously also unbounded.
We then show that there is a special ``Ville event'' $\mathcal E_\alpha$ on which the regret is $\ln\ln V_T +  \ln^2(1/\alpha)/V_T + \ln(1/\alpha)$ up to constants, where $V_T$ is a nonnegative, nondecreasing, cumulative variance process.
 What makes $\mathcal E_\alpha$ special is that if the data is stochastic (for example, conditionally zero mean and 1-sub-Gaussian), then $\mathcal E_\alpha$ holds with probability $1-\alpha$.  

We call this type of bound a \emph{conditional} regret bound, because it implies low regret on a \emph{large} set of paths, rather than all paths, where \emph{large} is quantified by being a high probability set under a suitable set of probability measures. In fact, we show that on $\mathcal E_0$ (a measure one set), every path \emph{eventually} has $\ln\ln V_T$ regret. 
To summarize at a high level, we prove that, deterministically:
\begin{quote}
\emph{Most} paths have $\ln\ln V_T$ regret, and \emph{almost all} paths eventually have $\ln\ln V_T$ regret.
\end{quote}

The above guarantee extends existing connections between wealth and regret in the online convex optimization or coin-betting literature, where path-wise regret bounds are standard (e.g., \citet{orabona2016coin,orabona2019modern, orabona2023tight,cutkosky2018black}), but the data, losses, or gradients in the prior work are generally assumed to be bounded. For unbounded data, it is impossible to prove low regret on all paths since a single extreme observation can blow up regret; our contribution is to prove an $O(\ln \ln V_T)$ regret on a large set of paths. We also point out that getting a $\ln V_T$ conditional regret bound is quite easy; we show how to do that (for a different mixture strategy) as a preliminary step for intuition. 

\paragraph{Contributions. } We now summarize key contributions of this work. 
\begin{enumerate}
    \item\label{unc_regbound} For observations $X_1, X_2, \dots$, and functions $\varprocess$ of $X_1, \dots, X_t$, we show that the classic sub-Gaussian process ($\exp\{\eta \sum_{i=1}^t X_i - \eta^2 \varprocess/2\}$ for $\eta \in \R$), when mixed over the parameters $\eta$ weighted using Robbins' prior, satisfies a deterministic regret bound when compared with $\max_\eta \exp\{\eta \sum_i X_i - \eta^2\varprocess/2\}$, where the $\max_\eta$ is unrestricted. This bound holds for all data sequences and does not require any assumptions (Eq.~\eqref{eq:pathwiseregretrobbins}, Theorem~\ref{th:lilregretV}). 
    
    \item \label{unc_lil} As a consequence, we show that on all paths in an event $\mathcal E_\alpha$ where the mixture process remains uniformly bounded by $\log(1/\alpha)$ for some $\alpha \in (0,1)$, the regret till time $T$ is $O(\ln^2(1/\alpha)/V_T + \log(1/\alpha) + \ln \ln V_T)$, for every $T$ (Eq.~\eqref{eq:lilregretonealpha}, Theorem~\ref{th:lilregretV}). We call such bounds that hold on a subset (for example, $\calE_\alpha$), rather than on all sequences, \emph{conditional regret bounds}. 
    
    \item\label{stoc_lil} If in Point~\ref{unc_lil} above, the data are assumed to be drawn from a distribution (not necessarily iid) such that the mixture processes is a non-negative supermartingale (Definition~\ref{def:subgaussian}), the ``Ville event'' $\calE_\alpha$ has measure at least $1-\alpha$. Thus, the pathwise regret bound from Point~\ref{unc_lil} holds for a \emph{large} fraction of paths (Theorem~\ref{th:lilregretV}). Importantly, the stochasticity is only used to quantify the size of $\calE_\alpha$ --- the regret bounds still hold deterministically on $\calE_\alpha$. Our deterministic bounds thus imply the high probability regret bound: $P(R_T = O(\ln\ln V_T + \ln^2(1/\alpha)/V_T + \log(1/\alpha)) \geq 1-\alpha$ simultaneously for a large class of measures $P$ that ensure $W_t$ is a sub-Gaussian process. 
    
    \item Without any assumptions on the data, we show that the conditional regret of $O(\ln\ln V_T)$ holds eventually (as opposed to for all $T$) on a larger set $\calE_0$, where the mixture process remains finite, but not necessarily uniformly bounded (Eq.~\eqref{eq:lilregretas}, Theorem~\ref{th:lilasregretV}). Interestingly, under the stochasticity assumptions form Point~\ref{stoc_lil}, when the mixture process is a non-negative supermartingale, $\calE_0$ has measure $1$. 
    While the previous bounds are akin to high probability regret bounds, this type of bound appears to be new in the literature: $P(R_T = O(\ln\ln V_T) \text{ eventually})=1$.
    Moreover, in this case, the mixture process acts as a witness for the (self-normalized) strong law of large numbers (SLLN) as well as an upper law of iterated logarithm (LIL), meaning that either self-normalized SLLN (and upper LIL) holds, or the mixture process gets unbounded (Theorem~\ref{th:lilasregretV}).
    
\end{enumerate}

\paragraph{Paper outline} 
In Section~\ref{sec:setup}, we formalize the setup and introduce the notation. Background material is reviewed in Section~\ref{sec:background}, and related work is discussed in Section~\ref{sec:literature}. As a warm-up, Section~\ref{sec:warmuplogTregret} analyzes logarithmic pathwise conditional regret bounds. Our main results are presented in Section~\ref{sec:lilregret}, where Theorem~\ref{th:lilregretV} establishes an unconditional regret bound valid for all data sequences. While the worst-case regret can grow linearly in $T$ (or $V_T$), we show that on a certain good event, the regret is only iterated-logarithmic eventually. Moreover, this good event is a measure $1$ set under appropriate stochasticity assumptions. 
Finally, Section~\ref{sec:conc} concludes with broader perspectives.

\section{Preliminaries: Background and Literature}\label{sec:preliminaries}

In this section, we formally introduce the setup, provide the necessary background, and discuss the relevant literature. 
\subsection{Setup}\label{sec:setup}
Let $X_1, X_2,\dots$ be a stream of observations taking values in $\R$, possibly unbounded. For each $t\in\N$, define the partial sum corresponding to the initial $t$ samples as $S_t:=\sum_{i=1}^t X_i$, and let $\varprocess$ be a function of $X_1, \dots, X_t$ that is nonnegative and increasing (usually strictly) in $t$. Further, $S_0=V_0=0$ and let $S_t = 0$ whenever $V_t = 0$, and we interpret $0/0=0$. 
For any fixed  $\eta\in\mathbb{R}$, let
\[f_t(\eta) := \eta S_t- \eta^2\varprocess/{2}. \numberthis\label{eq:ft} \]
Observe that $f_t(\cdot)$ is a concave function. Let 
\[\uncopt := \begin{cases}{S_t}/{\varprocess} &\text{ if } V_t > 0\\
0, &\text{otherwise},
\end{cases} \qquad \text{and} \qquad L_t^* := f_t(\uncopt) = \begin{cases}{S_t^2}/({2 \varprocess}), &\text{ if } V_t > 0\\
0, & \text{otherwise},
\end{cases}\]
denote the unconstrained maximizer for $f_t(\cdot)$, and the corresponding maximum value, respectively. 

Next, for any prior $\pi$ on $\eta$, define the mixture (wealth) process
\[
W_t := \int_{\mathbb R} e^{f_t(\eta)} \pi(\eta) d\eta,
\numberthis\label{eq:mixtureprocess}\]
with $W_0=1$; and $W_t := 1$ for $t$ such that $V_t = 0$. 
The regret of $W_t$ with respect to the maximum attainable wealth is defined as $$R_t := L_t^*-\ln W_t.$$
Note that $R_t=0$ whenever $V_t=0$, so it is without loss of generality that the reader may consider $V_t>0$ by default going forward. Finally, for $\alpha \in (0,1)$, consider the set 
\[ \calE_\alpha = \lrset{\sup\limits_{t \ge 1}~ \ln W_t \le \ln \frac{1}{\alpha}}, \numberthis\label{eq:ealpha} \]
on which the log-wealth $W_t$ remains bounded. We call $\calE_\alpha$ the ``Ville event''.

\begin{remark}
We do not make any assumptions about the data stream. In particular, it is an arbitrary sequence of real numbers. Hence, all the quantities defined above are deterministic. However, whenever the sequence $X_i$ is assumed to be drawn from some probability distribution, all capitalized variables defined above should be interpreted as random variables.
\end{remark}

The main focus of this work is to establish regret bounds for the mixture wealth process $W_t$, constructed using a mixing distribution $\pi$ originally introduced in \cite{robbins1968iterated}. We show that this classical stochastic construction can be reinterpreted within an adversarial framework: specifically, $W_t$ admits deterministic, path-wise regret guarantees, which have an interesting nontrivial interpretation. The worst-case regret is unbounded (linear in $t$ or $V_t$); but we can prove that one attains low regret on every path in $\mathcal E_\alpha$. What makes $\mathcal E_\alpha$ special is that under natural stochastic assumptions (implied by the underlying mixture processes), the event holds with high probability. So we obtain a deterministic low regret bound on every path in a high probability event. In fact, we even show that we achieve a small regret \emph{eventually} on almost every path (i.e.\ on every path in measure $1$ set $\mathcal E_0$). In other words, the set of paths on which regret is unbounded is a measure zero set (under many natural families of distributions specified later). 

Our work thus connects worst-case regret analysis with stochastic betting approaches developed in game-theoretic statistics. To place our results in context, we first review the necessary background that will be used throughout. We also refer the reader to Appendix~\ref{app:bettinggame} for a detailed discussion of the betting game considered in this work, including the performance criteria for the betting strategies and the associated notion of regret. 


\subsection{Background}\label{sec:background}
We begin with notions that characterize the behavior of stochastic processes, such as sub-Gaussian tails, the strong law of large numbers (SLLN), and the law of the iterated logarithm (LIL). We then review the framework of betting and wealth processes, which will serve as useful tools in relating the adversarial and stochastic betting via our conditional regret bounds. After some formal setup, we present a key  definition of \emph{sub-Gaussian processes} from  \cite{victor2004self, howard2020time}.

When dealing with stochastic data, we will work on a filtered measurable space $(\Omega,\mathcal F)$, where $\mathcal F = (\mathcal F_t)_{t \geq 0}$ is a filtration (an increasing sequence of sigma algebras). We will also consider a family $\Pi$ of probability measures defined on $(\Omega,\mathcal F)$. 

For this paper, the underlying $\Omega$ is typically $\R^\infty$, where we observe a single real-valued observation $X_t$ at each time $t$ so that $\mathcal F_t = \sigma(X_1,\dots,X_t)$, with $\mathcal F_0$ being trivial. 
A stochastic process $M$ is a sequence of random variables adapted to $\mathcal F$, meaning that $M_t$ is $\mathcal F_t$-measurable. For $P \in \Pi$, we say that $M$ is $P$-integrable if $\mathbb E_P[|M_t|]<\infty$ for all $t$. A $P$-supermartingale is a $P$-integrable stochastic process $M$ such that $\mathbb E_P[M_t | \mathcal F_{t-1}] \leq M_{t-1}$, $P$-almost surely. Given  $\mathcal P \subset \Pi$, we say that $M$ is a $\mathcal P$-supermartingale if $M$ is a $P$-supermartingale for every $P \in \mathcal P$. If $P$ or $\mathcal P$ are understood from context or not important, we just call $M$ a supermartingale.

\begin{definition}[Sub-Gaussian process]\label{def:subgaussian}
A stochastic process $(S_t)_{t \geq 0}$ is said to be a sub-Gaussian process with a nonnegative and nondecreasing variance proxy  $(\varprocess)_{t \geq 0}$ if $(\exp(\eta S_t - \eta^2 \varprocess/2))_{t \geq 0}$ is a (nonnegative) super-martingale for every $\eta\in\R$. Typically, we take $S_0=V_0=0$ by default.
\end{definition}


$S_t$ is often simply a sum process $\sum_{i=1}^t X_t$.
We recall in Appendix~\ref{app:subGaussian} that the above condition is satisfied (for different $V_t$) when $\mathcal P$ contains distributions that are (a) centered sub-Gaussian, or (b) symmetric around 0, or (c) have mean zero and finite variance. The condition thus encompasses both light-tailed (case (a)) and heavy-tailed (cases (b) and (c)) settings. Typically, $V_t$ will be strictly positive and strictly increasing. However, we do not assume it because in case (b) above, we have $V_t = \sum_{i \leq t} X_i^2$, and if the distribution of $X_1$ (for example) is symmetric with a point mass at 0 (eg: it equals 0 with probability half, and a standard Gaussian with probability half), then it is possible for $X_1=0$ to occur with positive probability, and then  $S_1=V_1=0$. Outside of these sorts of edge cases, $V_t$ will indeed be strictly increasing and positive for $t \geq 1$. Also, typically $V_t$ will increase to infinity as $t \to \infty$, but we will only assume that when needed to derive limiting statements.

\paragraph{Strong law of large numbers (SLLN) and iterated logarithm (LIL). } SLLN and LIL describe the asymptotics of stochastic processes. For example, when $X_i$ are iid from some distribution with mean $0$, the SLLN guarantees that $S_t/t \to 0$ almost surely as $t\to\infty$, while if they also have finite variance $\sigma^2$, then the LIL refines this as below: 
\[ \limsup\limits_{t\to\infty} \frac{|S_t|}{\sigma\sqrt{2t\ln\ln t}} = 1, \quad \text{ almost surely.} \]
These results also extend to self-normalized processes, which are especially useful when variances are unknown or data are heavy-tailed \citep{de2009self}; one can derive a self-normalized SLLN and LIL involving $S_t/V_t$ and $S_t/\sqrt{V_t \log\log V_t}$.
Our versions of self-normalized SLLN and LIL for sub-Gaussian processes are in Section~\ref{sec:lilregret}.



\paragraph{Martingales, betting, and wealth processes.} 
There is an equivalence between non-negative (super)martingales and bets in a (conservative) fair game. A wealth process $(W_t)_{t \geq 0}$ is a nonnegative adapted process that satisfies $W_0=1$. Such processes can be interpreted as the capital of a gambler who starts with a unit initial wealth $(W_0 = 1)$, and at each time $t$ bets a fraction of the current wealth $W_{t-1}$ on the next outcome. The gambler's wealth at time $t$ is $W_t$. If the bets are made in a fair (conservative) game, then the wealth process $(W_t)_{t\ge1}$ is a non-negative (super)martingale with respect to the natural data filtration. Ville’s inequality then yields the following time-uniform guarantee: $P[\sup_t W_t \ge \frac{1}{\alpha}] \le \alpha$. 

Conversely, given a non-negative martingale $(M_t)_{t\ge 1}$ with $M_0 = 1$, it can be viewed as a wealth process generated by a sequence of bets. At time $t$, the per-round bet is an e-variable \citep[\S 1]{ramdas2024hypothesis} that is given by $\frac{M_t}{M_{t-1}}1(M_{t-1}>0)$, with $0/0=0$. The martingale property ensures that the game is fair.  This perspective underlies developments in game-theoretic probability~\citep{shafer2005probability,shafer2019game} and game-theoretic statistics \citep{shafer2021testing,ramdas2023game,ramdas2024hypothesis}, where nonnegative martingales are also referred to as wealth processes, and stochastic coin-betting algorithms \citep{orabona2016coin}. We refer the reader to Appendix~\ref{app:bettinggame} for a detailed discussion on the betting game that we consider in this work.  

\subsection{Literature Review}\label{sec:literature}
In the online learning and information theory literature, \cite{cover1991universal, cover2002universal} prove an $O(\log t)$ deterministic regret bound against all sequences of observations from a bounded domain, say $[0,1]$, for their mixture strategy. Several later works in online learning literature propose different mixture wealth processes that improve the constant terms in the path-wise regret bound of $O(\log t)$, that holds for any sequence of bounded data. 

\cite{cutkosky2019matrix} use wealth/betting arguments to derive a path-wise regret bound in an adversarial online learning setting. However, they require that the gradients are bounded in a certain dual norm. \citet{koolen2015second} prove a deterministic, pathwise regret bound that is $O(\sqrt{V_t\ln\ln V_t})$ for a certain variance process $V_t \le t$. However, this result crucially relies on bounded losses. 

\citet{orabona2023tight} leverage the wealth-regret connection, but their primary aim is to construct time-uniform confidence sets for the mean of bounded data using the wealth process of certain betting strategies. In their work, they implicitly prove iterated logarithmic conditional regret bounds for bounded data for a wealth process different from $W_t$. By contrast, we extend the analysis to unbounded data and also prove regret bounds for $W_t$ that hold for arbitrary sequences of observations. 

\citet[Lemma F.1]{agrawal2021optimal} prove a deterministic regret bound of $O(\ln t)$ for unbounded data, that holds for an arbitrary sequence of observations. However, they define regret with respect to a constrained optimizer. In our work, we instead consider regret with respect to the global optimizer and prove a deterministic regret bound for unbounded data, which we show is eventually $O(\ln\ln t)$ on a measure $1$ set for stochastic data. 

In the coding literature, \citet{rissanen2002fisher} showed that if the parameter space of an exponential family is restricted to a compact subset of the interior of the parameter space with non-empty interior, then the minimax regret is $O(\log t)$. We refer the reader to \cite{grunwald2007minimum} for more details. We do not make any such assumptions, and instead, prove a conditional regret bound of $O(\ln\ln t)$, that eventually holds on a measure $1$ set (instead of against all sequences).

There are a few works in the literature that derive inequalities for martingales (concentration/bound on moments, etc.) from certain deterministic counterparts. For example, \cite{acciaio2013trajectorial} derive Doob's $L^p$ maximal inequalities (bounds on $p^{th}$ moments of a martingale) as a consequence of elementary deterministic inequalities that hold for all sequences. \cite{beiglbock2014martingale} show that a general martingale inequality can be reduced to certain deterministic inequalities in a small number of variables. \cite{beiglbock2015pathwise} provide a proof of the Burkholder-Davis-Gundy inequalities as a consequence of elementary deterministic inequalities. \cite{gushchin2014pathwise} derives Doob's maximal inequality on the probability of a martingale exceeding a level, from deterministic counterparts. These works prove probabilistic martingale inequalities via deterministic statements. Despite also connecting pathwise statements and martingale inequalities, our contributions are orthogonal: we prove deterministic regret bounds for a particular mixture method on unbounded data. 

\citet{clerico2025confidence} develop a method to construct time-uniform confidence sets for parameters of a generalized linear model, via reduction to the online learning framework. They prove certain path-wise data-dependent regret bounds that hold with probability $1$, and use them to construct confidence sequences. We establish deterministic regret bounds for unbounded data, and show that under stochastic assumptions, these bounds improve to $O(\ln \ln t)$ on a set of measure $1$. We are not aware of any work that proves such rates for unbounded observations, on a set of measure $1$.

Finally, \cite{rakhlin2017equivalence} highlight the tight connections between martingale concentration inequalities and regret analysis. They show an equivalence between tail inequalities on the supremum of a collection of martingales and certain path-wise regret bounds. However, our contributions appear orthogonal to theirs; there is no mention in their work of Ville's inequality, mixture wealth processes, etc. Nevertheless, our work can be considered within the same theme of connecting martingales and regret. 

Shafer and Vovk \citep[Chapter~5]{shafer2005probability} prove a game-theoretic version of the LIL that holds path-wise and applies to arbitrary forecasting protocols. However, their framework implicitly restricts the sequence of outcomes to be bounded, with possibly time-dependent bounds (see, \citet[Theorems 5.1, 5.2]{shafer2005probability}). Our results are complementary: we handle unbounded data via self-normalization, obtain explicit regret bounds valid for all $t$, and construct the supermartingale that certifies our version of the SLLN and LIL.

All the above-cited works span different sub-fields. A complete picture that relates these is not entirely clear to us, and forming such a complete picture is a direction for future work.

\section{Warm-up: Logarithmic Conditional Regret}\label{sec:warmuplogTregret}
To get some intuition, we begin with a simple setting when the prior $\pi$ on $\eta$ is a Gaussian distribution with mean $0$  and variance $\sigma^2_0$, for some $\sigma_0 > 0$. Then, the mixture (wealth) process from~\eqref{eq:mixtureprocess} becomes
\[
W_t = \frac{1}{\sqrt{1+\sigma_0^2 V_t}}e^{\frac{\sigma_0^2 S_t^2}{2(1+\sigma_0^2 V_t)}},
\numberthis\label{eq:mixV}\]
and as earlier, $W_t = 1$ for $t$ such that $V_t = 0$. 
Theorem~\ref{th:ville_gaussianmixture} below shows that the mixture $W_t$ incurs a regret of $O(\ln \varprocess + \ln(1/\alpha))$ on the set $\calE_\alpha$, defined in~\eqref{eq:ealpha}. 

\begin{theorem}\label{th:ville_gaussianmixture}
For all $t\ge 0$,
\[
R_t
 =  \frac{1}{2}\ln\big(1+\sigma_0^2 V_t\big)
 +  \frac{S_t^2}{2V_t (1+\sigma_0^2 V_t)}. \]
 (Recall 0/0=0.)
 Further, for any $v_0>0$, the following holds on $\calE_\alpha$:  
\[
\forall t\ge 1, \quad  \text{either} \quad V_t\le v_0 \quad \text{or} \quad R_t  \le  \Big(\tfrac12 + \tfrac{1}{2\sigma_0^2 v_0}\Big) \ln\big(1+\sigma_0^2 V_t\big) +  \frac{\ln(1/\alpha)}{\sigma_0^2 v_0}.\numberthis\label{eq:boundrtgaussianmix}
\]
Moreover, if the data are drawn from any distribution $P$ such that $(S_t)_{t\ge 1}$ is a sub-Gaussian process with variance proxy $(V_t)_{t\ge 1}$, then 
$P[\calE_\alpha] \ge 1-\alpha$.
\end{theorem}
The proof of the above theorem proceeds by first establishing an expression for the unconditional regret, which corresponds to the first expression in the theorem statement. Next, on the set $\calE_\alpha$,  $\ln W_t$ is bounded. Using this bound in the expression for $R_t$ after appropriate rearrangements, we get a conditional bound on $R_t$ on the set $\calE_\alpha$, in terms of $\varprocess$ and $S_t$. The expression in~\eqref{eq:boundrtgaussianmix} then follows by bounding certain occurrences of $\varprocess$ by $v_0$. The proof is straightforward, and we present it in Appendix~\ref{app:proof_warmuplogTregret} for completeness. 

\begin{corollary}
    For the special case when $X_i$ are drawn iid from a standard Gaussian distribution, $S_t$ is a sub-Gaussian process with variance proxy $\varprocess = t$ (see Appendix~\ref{app:subGaussian}). Assuming $\sigma_0 = 1$ for simplicity, in this case, $R_t = \frac{1}{2}\ln(1+t) + \frac{S^2_t}{2t(1+t)}$ path-wise. Moreover, on $\calE_\alpha$, which occurs with probability at least $1-\alpha$, we have $R_t \le \ln(1+t) + {\ln(1/\alpha)}$, for all $t\ge 1$.
\end{corollary}

In fact, it follows from the following theorem that $O(\ln V_t)$ regret for the Gaussian-mixture wealth holds eventually on every path on a set of measure $1$.

\begin{theorem}\label{th:aslogregret}
Let $\sigma_0 = 1$ for simplicity. Consider the event
\[
\calE_0 := \lrset{ V_t \uparrow \infty;~ \limsup_t W_t < \infty}.
\]
We have that on $\calE_0$, 
$$R_t \leq \ln (1+\varprocess) + o(1) \text{ eventually},$$  
i.e.\ on every path in $\calE_0$, and for every constant $d$, there exists an $n$ such that for all $t \geq n$, $R_t \leq  \ln (1+\varprocess) + d$. 
In particular, dividing by $\varprocess$ and taking the limit as $t \to \infty$, we get that $S_t /\varprocess \to 0$ path-wise on $\calE_0$. Finally, if the data are drawn from a distribution $P$ such that $(S_t)_{t\ge 1}$ is sub-Gaussian with variance proxy $(\varprocess)_{t\ge 1}$, then $W_t$ is a nonnegative super-martingale, immediately implying that $\calE_0$ is a measure 1 event, thus implying the (self-normalized) SLLN. Contrapositively, if there is a path on which $S_t /\varprocess \to 0$ does not hold, then $\limsup_t W_t = \infty$ on that path.
\end{theorem}


Ville's martingale theorem~\citep{ville1939etude} states that fixing a measure $P$, for every event $A$ such that $P(A)=0$, there exists a nonnegative $P$-martingale $M^P$ that increases to $\infty$ on $A$. But this $M^P$ is potentially different for every $P$. Here, we have a single super-martingale $Z=(W_t)_{t\ge 1}$ that serves as a \emph{witness} for the SLLN for an entire class of distributions $\mathcal P$ (eg, zero mean, 1-sub-Gaussian distributions).  In general, martingales and even super-martingales may not suffice as witnesses for the strong law for composite $\mathcal P$, or more generally for a composite generalization of Ville's theorem; for example, \cite{ruf2023composite} showed that in general one must employ e-processes instead (see \cite{ruf2023composite} for a definition). A general characterization of when super-martingales suffice is open; the above result provides an interesting example.

Moreover, Theorem~\ref{th:aslogregret} provides an explicit witness super-martingale of the violation of the strong law, meaning that if the strong law is violated on some path, $\limsup_t W_t = \infty$ on that path. 
\begin{proof}
   Since  
   \[
   W_t := \int_{\mathbb{R}} e^{f_t(\eta)} \pi(\eta) d\eta = \frac{e^{\frac{S_t^2}{2 (1+\varprocess)}}}{\sqrt{1+\varprocess}},
   \]
   we see that $\calE_0$ implies that $S_t^2/(1+\varprocess) < \sqrt{\varprocess}$  eventually (where the right hand side can be picked anything asymptotically larger than $\ln (1+\varprocess)$). Since
   \[
R_t
 = L_t^*-\ln W_t
 = \frac{1}{2}\ln (1+\varprocess)
 + \frac{S_t^2}{2\varprocess (1+\varprocess)},
\]
   we then see that $\calE_0$ implies $R_t \leq \ln(1+\varprocess) + o(1)$.
\end{proof}

\section{Main Result: Iterated Logarithmic Conditional Regret}\label{sec:lilregret}
Recall $S_t$, $V_t$, $f_t$, $\uncopt$, $L^*_t$, and $\calE_\alpha$ from Section~\ref{sec:setup}. In this section, we establish a regret bound for Robbins' mixture strategy \citep{robbins1968iterated}, conditioned on the set $\calE_\alpha$ from~\eqref{eq:ealpha} on which the log of the mixture wealth remains bounded. To achieve this, as in Section~\ref{sec:warmuplogTregret}, we first prove a path-wise regret bound that holds for all data sequences $X_1, X_2, \dots$. In general, this bound can be arbitrarily large (at least linear in $t$). However, we show that on the set $\calE_\alpha$, where the mixture process $W_t$ is restricted to be uniformly bounded, the regret is also small. In fact, for the mixture wealth using Robbins' heavy-near-zero prior, we show that the bound on regret till time $t$, for every $t$, is $O(\ln \ln V_t)$. 

To formalize, for $f_t$ defined in~\eqref{eq:ft}, in this section, we consider the mixture wealth process $W_t$ given by 
\[ W_t = \frac{1}{Z_0}  \int\limits_{-1}^1 \frac{e^{f_t(\eta)}}{|\eta| (\ln\tfrac{c}{|\eta|}) (\ln\ln\tfrac{c}{|\eta|})^2} d\eta, \numberthis\label{eq:robbinsmixwealth} \]
which is obtained by using the following mixture $\pi$ in~\eqref{eq:mixtureprocess}: for $c \ge 6.6 e$ and $\eta\in[-1,1]$, 
\[\pi(\eta)=\frac{1}{Z_0} \frac{{\mathbf 1}_{|\eta|\le 1}}{|\eta| (\ln\tfrac{c}{|\eta|}) (\ln\ln\tfrac{c}{|\eta|})^2}.\]
It equals $0$ outside $[-1,1]$. Here, $Z_0$ is the normalization constant 
\[Z_0 = \int_{-1}^1 \frac{d\eta}{|\eta| (\ln\tfrac{c}{|\eta|}) (\ln\ln\tfrac{c}{|\eta|})^2} = \frac{2}{\ln\ln c}.\]

\begin{theorem}\label{th:lilregretV}
    For all $t\ge 1$ and $\rho \in (0,\frac{1}{4})$,
\begin{align*}
    R_t \le \begin{cases}
        \frac{1}{2} + \ln\tfrac{2}{\ln\ln c} + \ln\ln(c\sqrt{1+\varprocess}) + 2 \ln\ln\ln(c\sqrt{1+\varprocess}), & \text{ if }\frac{|S_t|}{\varprocess}\le \frac{1}{\sqrt{1+\varprocess}} \\
        \frac{1}{2} + \ln \tfrac{2}{\ln\ln c} +\ln\left(\frac{|S_t|}{\sqrt{\varprocess}}\sqrt{1+\frac{1}{\varprocess}}\right) \\ 
        \qquad \qquad + \ln\ln(c\sqrt{1+\varprocess}) + 2\ln\ln\ln(c\sqrt{1+\varprocess}), & \text{ if } \frac{1}{\sqrt{1+\varprocess}}<\frac{|S_t|}{\varprocess} \le 1\\
        \frac{\varprocess}{2}\left( \frac{|S_t|}{\varprocess} - 1 + \rho  \right)^2 - \ln\left( \frac{\rho\ln \ln c}{ \ln(c/{1-\rho}) (\ln\ln(c/1-\rho))^2}  \right) , &\text{ if }  1 < \frac{|S_t|}{\varprocess}. 
    \end{cases}\numberthis\label{eq:pathwiseregretrobbins}
\end{align*}
Moreover, on $\calE_\alpha$, there exist a universal constant $C>0$ such that
\[\forall t\ge 1, \quad R_t \leq C\lrp{1 + \frac{1+\ln^2(1/\alpha)}{V_t} + \ln(1/\alpha) +  \ln \ln\left(c\sqrt{1+\varprocess}\right)}. \numberthis\label{eq:lilregretonealpha} \]
Furthermore, if the data are drawn from any distribution $P$ such that $(S_t)_{t\ge 1}$ is a sub-Gaussian process with variance proxy $(V_t)_{t\ge 1}$, then $P[\calE_\alpha] \geq 1-\alpha$.
\end{theorem}


Observe that the bound in~\eqref{eq:pathwiseregretrobbins} holds along \emph{all} data sequences. The three cases in this bound capture the behavior of the unrestricted optimizer for the log-wealth. This bound, especially when ${|S_t|}/{\varprocess} > 1$ (third branch), can be arbitrarily large. To get some intuition, let us consider the special case of $S_t= \sum_i X_i$ and $V_t = t$. The third branch then corresponds to ${|S_t|}/{t} > 1$. If the data are not stochastic or are stochastic but have a mean greater than $1$, then ${|S_t|}/{t}$ will likely be greater than $1$. Hence, the third branch in~\eqref{eq:pathwiseregretrobbins} is likely to occur, and the bound on the regret will be linear in $t$ (at least). On the other hand, if the data are iid from a 1-sub-Gaussian distribution (for example, $N(0,1)$), then both SLLN and LIL hold. In particular, $S_t/t \rightarrow 0$ almost surely, and $S_t/\sqrt{t} \approx \sqrt{\ln \ln t}$ infinitely often. These facts suggest that the third branch of the regret bound can occur only finitely many times, and only leads to at most a constant regret (possibly a function of $\alpha$). In this setting, since the first branch eventually applies, it yields $O(\ln\ln t)$ regret. Moreover, by LIL, the second branch occurs infinitely often; however, the third term within this branch is $O(\ln\ln\ln t)$ (by LIL), which again leads to regret of $O(\ln\ln t)$ in this case as well.

Note that in the above discussion, we assumed the data to be stochastic only to develop some intuition about the bound in~\eqref{eq:pathwiseregretrobbins}. However, it does not require any stochasticity assumption, and instead holds path-wise. Next, in~\eqref{eq:lilregretonealpha} we show that on the set where the data is well-behaved in the sense that the wealth process $W_t$ is uniformly bounded from above, the path wise regret is small, and is in fact, $O(\ln\ln V_t)$ at time $t$. Even the path-wise bound in~\eqref{eq:lilregretonealpha} does not require any stochasticity assumption. Stochasticity is only used in the last part of the theorem, which shows that if the data is stochastic and satisfies certain conditions, the set $\calE_\alpha$ where~\eqref{eq:lilregretonealpha} holds is indeed a large set (of probability at least $1-\alpha$).

Finally, observe from~\eqref{eq:lilregretonealpha} that if $V_t \ge \log(1/\alpha)$, then on $\calE_\alpha$, $R_t = O(\log(1/\alpha) + \ln\ln(1+V_t))$.

\paragraph{Confidence sequences (CS) from conditional regret bound.} Next, under the stochasticity assumption, we observe that our regret bound in~\eqref{eq:pathwiseregretrobbins} immediately implies a $(1-\alpha)$-CS (a sequence of confidence intervals with a time-uniform coverage guarantee of $1-\alpha$) for the mean of the unknown data-generating distribution. To see this, let $B_t$ denote the right hand side of~\eqref{eq:pathwiseregretrobbins}, and recall that $R_t = {S^2_t}/({2V_t}) - \log W_t$. Then, from~\eqref{eq:pathwiseregretrobbins}, the following holds path wise on $\calE_\alpha$: 
\[ \forall t\ge 1, \qquad  \frac{S^2_t}{2V_t} - \log\frac{1}{\alpha} \le R_t \le B_t. \]
If the data is drawn from a distribution $P$ such that $(S_t)_{t\ge 1}$ is a sub-Gaussian process with variance proxy $(V_t)_{t\ge 1}$, then the above further implies
\[ P\lrs{\forall t\ge 1,~ \frac{S^2_t}{2V_t} \le B_t + \ln\frac{1}{\alpha}} \ge 1-\alpha. \numberthis\label{eq:regrettocs}\]
Finally, in the special case that the data are conditionally $\sigma_t$-sub-Gaussian with mean $\mu_t$ (for some predictable $\mu_t,\sigma_t$), we substitute $S_t = \sum_{i \leq t} (X_i - \mu_i)$ and $V_t = \sum_{i \leq t} \sigma_i^2$ to get that 
\[
P \lrs{\forall t\ge 1,~ \overline\mu_t \in \hat\mu_t \pm \left\{\sqrt{\frac{2V_t(B_t + \ln\frac{1}{\alpha})}{t^2}} \right\}} \ge 1-\alpha,\numberthis\label{eq:cs}
\]
where $\bar\mu_t =\sum_{i \leq t}  \mu_i/t$ and $\hat \mu_t = \sum_{i \leq t}  X_i/t$.
To see that the shape of this CS is optimal, we refer the reader back to Eq.~\eqref{eq:lilregretonealpha}.
In fact, later in Theorem~\ref{th:lilasregretV}, we show that $B_t = \ln\ln \varprocess (1+o(1))$ eventually on $\calE_\alpha$. Thus, the CS in~\eqref{eq:regrettocs}, obtained using regret bound in~\eqref{eq:pathwiseregretrobbins}, recovers the exact LIL-CS. 

\paragraph{Conditional regret bounds from confidence sequences (CS). } 
We now explain how, in martingale-based CS constructions, one can also read off a conditional regret bound similar to those in~\eqref{eq:lilregretonealpha}. To this end, consider again the special case where the data is drawn from a distribution $P$ such that $(S_t)_{t\ge 1}$ is a sub-Gaussian process with variance proxy $(V_t)_{t\ge 1}$. Moreover, for some predictable $\mu_t$,  let $S_t = \sum_{i\le t} (X_i - \mu_i)$, and suppose we have the following CS with $\bar{\mu}_t = \sum_{i\le t} \mu_i / t$, $\hat{\mu}_t = \sum_{i\le t}X_i / t$, a positive constant $c_1$, and $\alpha \in (0,1)$:
\[ P\lrs{ \forall t\ge 1,~ \abs{ \bar{\mu}_t - \hat{\mu}_t} \le \sqrt{ \frac{2V_t\lrp{c_1 \ln\ln (e+V_t) + \ln\frac{1}{\alpha}}}{t^2} }  } \ge 1-\alpha. \numberthis\label{eq:givencs} \]
One frequently used approach for arriving at inequalities of the form~\eqref{eq:givencs} is to threshold a non-negative (mixture) super-martingale (or equivalently, a wealth process), say $W_t$, at a level $\tfrac{1}{\alpha}$, followed by an application of Ville's inequality \citep{darling1967confidence,robbins1970boundary}. 

As earlier, let $\mathcal E_{\alpha} := \{\sup_t \ln W_t \le \ln(1/\alpha)\}$. While directly thresholding a wealth process $W_t$ (as in $\mathcal E_\alpha$) would give an implicit CS for mean, to get an explicit CS of the form in~\eqref{eq:givencs}, these martingale-based approaches typically proceed by deriving a lower bound for $\ln W_t$ that is relatively tight on the set $\mathcal E_\alpha$, and thresholding this lower bound process, instead. In particular, these methods can be seen to proceed by implicitly establishing an inequality like the following: on $\mathcal E_\alpha$,
\[ \ln W_t \ge \frac{S^2_t}{2 V_t} - c_1 \ln \ln(e+ V_t), \quad \forall t \ge 1. \]
implying
\[ \lrset{\sup_t \ln W_t \le \ln\frac{1}{\alpha}} \subseteq \lrset{\sup_t \lrp{\frac{S^2_t}{2 V_t} - c_1 \ln \ln (e+V_t)  } \le \ln \frac{1}{\alpha} }, \]
hence~\eqref{eq:givencs}. But, the above inequality also implies the following on $\mathcal E_\alpha$:
\[ R_t = \frac{S^2_t}{2V_t} - \ln W_t \le c_1 \ln \ln (e+V_t), \quad \forall t\ge 1, \]
yielding a conditional regret bound of $O(\ln\ln(e+ V_t))$ on $\mathcal E_\alpha$, in the same spirit as~\eqref{eq:lilregretonealpha} with $C = c_1$. We remark that our direct approach in Theorem~\ref{th:lilregretV} yields tighter path-wise regret bounds that hold unconditionally (Eq.~\eqref{eq:pathwiseregretrobbins}), and lead to exact upper LIL, as shown later in Theorem~\ref{th:lilasregretV}.

\paragraph{Proof sketch. } The proof of Theorem~\ref{th:lilregretV} proceeds in three steps. In Step 1 (Section~\ref{sec:step1}), we prove the path-wise bound on $R_t$ in~\eqref{eq:pathwiseregretrobbins}. In Step 2 of the proof (Section~\ref{sec:step2}), we show that on the set $\calE_\alpha$ where the mixture-wealth process is uniformly bounded across $t$, the bounds in~\eqref{eq:pathwiseregretrobbins} are $O( \ln(1/\alpha) + \frac{1}{\varprocess} + \frac{1}{\varprocess}\ln^2(1/\alpha) + \ln\ln \varprocess)$, for every $t$. Finally, Step 3 (Section~\ref{sec:step3}) shows that when the data is stochastic with appropriate assumptions, the conditional regret bound from Step 2 holds for a large number of sequences. In particular, the set $\calE_\alpha$ occurs with probability at least $1-\alpha$. Note that only the third step of the proof requires the stochasticity assumption on data. We present a complete proof of Theorem~\ref{th:lilregretV} in Section~\ref{sec:proof_thlilregretV}. 

\begin{remark}
    Observe that the Ville's event $\calE_\alpha$ does not prevent the maximum hindsight wealth $\frac{S^2_t}{2\varprocess}$ from being unbounded. In fact, it can drift away from the log of the mixture wealth at the rate of $O(\ln\ln V_t)$. Since the mixture wealth is also uniformly bounded on the Ville's event, this means that the maximum wealth can increase to $\infty$ at $O(\ln\ln V_t)$ rate. 
\end{remark}

Theorem~\ref{th:lilregretV} above proves a path-wise regret bound (Eq.~\eqref{eq:pathwiseregretrobbins}), as well as a conditional regret bound of $O(\ln \ln \varprocess)$ (Eq.~\eqref{eq:lilregretonealpha}) that holds for every $t$ on a set $\calE_\alpha$. If the data is stochastic, then $\calE_\alpha$ is large. In the following theorem, we show that the $O(\ln\ln \varprocess)$ conditional regret bound holds eventually on a larger set $\calE_0$ where the process $W_t$ remains finite, but possibly not uniformly bounded. This further implies that 
on $\calE_0$, $|S_t|/\varprocess \to 0$ and $\limsup_t |S_t|/\sqrt{\varprocess\ln\ln\varprocess} \le \sqrt{2}$. In other words, either both of these aforementioned convergences hold or the process $W_t$ explodes. In particular, if the data are stochastic with appropriate assumptions, the set $\calE_0$ is a set of measure $1$, and we conclude that the process $W_t$ acts as a witness for the self-normalized SLLN and (upper) LIL.

\begin{theorem}\label{th:lilasregretV}
Consider the event
\[
\calE_0 := \lrset{ V_t \uparrow \infty;~ \limsup_t W_t < \infty}.
\]
We have that on $\calE_0$, 
\[R_t \leq \ln\ln \varprocess (1+ o(1)) \text{ eventually},\numberthis\label{eq:lilregretas}\]  
{i.e.\ on every path in $\calE_0$, and for every $\epsilon > 0$, there exists $n$ such that for all $t \geq n$, $R_t \leq  (1+\epsilon)\ln\ln \varprocess$.} In particular, dividing by $\varprocess$ and taking the limit as $t \to \infty$, we get that $|S_t|/\varprocess \to 0$ path-wise on $\calE_0$. Next, dividing by $\ln\ln V_t$ and taking limit as $t \to \infty$, we get  $$\limsup_t \frac{|S_t|}{\sqrt{2 V_t \ln \ln V_t}} \le 1$$ path-wise on $\calE_0$. As a corollary, if the data are drawn from a distribution such that $(S_t)_{t\ge 1}$ is a sub-Gaussian process with variance $(\varprocess)_{t\ge 1}$, then $W_t$ is a nonnegative super-martingale, immediately implying that $\calE_0$ is a measure 1 event, thus implying the self-normalized SLLN and LIL. 
\end{theorem}

As in Theorem~\ref{th:aslogregret}, Theorem~\ref{th:lilasregretV} demonstrates a deterministic statement of the form: either $R_t = O(\ln\ln V_t)$, or the wealth process $W_t$ explodes. In case of stochastic data, it provides an explicit super-martingale witnessing the violation of the SLLN and LIL: if either of these is violated on some path, $\limsup_t W_t = \infty$ on that path.

\begin{proof}
Since $R_t = \frac{S^2_t}{2 \varprocess} - \ln W_t$, from~\eqref{eq:pathwiseregretrobbins} we get a lower bound on $W_t$. Further, using $\ln(x) \le x-1$ and completing the squares, we get the following path-wise lower bound on $W_t$:
\begin{align*}
    \ln W_t \ge \begin{cases}
        \frac{S_t^2}{2V_t} - \frac{1}{2} - \ln\tfrac{2}{\ln\ln c} - \ln\ln(c\sqrt{1+\varprocess}) - 2\ln\ln\ln(c\sqrt{1+V_t}), & \text{ if }\frac{|S_t|}{\varprocess}\le \frac{1}{\sqrt{1+\varprocess}} \\
        \frac{1}{2}\lrp{\! \frac{|S_t|}{\sqrt{\varprocess}}  -  \sqrt{1 + \frac{1}{V_t}} }^2 - \ln\tfrac{2}{\ln\ln c} - \frac{1}{2V_t}\\ 
        \qquad\qquad - \ln\ln\left(c\sqrt{1+\varprocess}\right) - 2\ln\ln\ln(c\sqrt{1+\varprocess}), & \text{ if } \frac{1}{\sqrt{1+\varprocess}}<\frac{|S_t|}{\varprocess} \le 1\\
        \frac{S_t^2}{2V_t} - 
        \frac{\varprocess}{2}\left( \frac{|S_t|}{\varprocess} - 1 + \rho  \right)^2 + \ln\left( \frac{\rho\ln\ln c}{\ln({c}/{1-\rho})(\ln\ln({c}/{1-\rho}))^2}  \right) , &\text{ if }  1 < \frac{|S_t|}{\varprocess}. 
    \end{cases}
    \numberthis\label{eq:lilboundonzt}
\end{align*}


We now argue that on $\calE_0$, the third branch occurs only finitely many times.
On the third branch, since $|S_t| > \varprocess$, we can further lower bound as
\[ \ln W_t \ge \varprocess\lrp{\frac{1-\rho^2}{2}} + \ln\left( \frac{\rho\ln\ln c}{\lrp{ \ln\tfrac{c}{1-\rho} }\lrp{\ln\ln\tfrac{c}{1-\rho}}^2}  \right). \]
Recall that on $\mathcal E_0$, $V_t \uparrow \infty$ and $\limsup_t W_t <\infty$. From this and the inequality above it, we see that on $\calE_0$, if $\limsup_t \frac{|S_t|}{\varprocess} > 1$ (i.e., the third branch occurs infinitely often), then $\limsup_t\ln W_t = \infty$, which leads to a contradiction. 

Thus, we only focus on the first two cases (i.e., $t$ such that $\frac{|S_t|}{\varprocess} \le 1$), henceforth. Now, on $\calE_0$, dividing both sides of~\eqref{eq:lilboundonzt} by $\varprocess$ and taking limit as $t\to\infty$ we get $|S_t|/\varprocess \to 0$. Similarly, dividing by $\ln\ln \varprocess$ instead, and taking the limit as $t\to\infty$, we get 
\[\limsup_t \frac{|S_t|}{\sqrt{2\varprocess\ln\ln \varprocess}} \le 1. \numberthis\label{eq:snlil} \]

Finally, from~\eqref{eq:lilboundonzt}, we also have the following eventually on $\calE_0$:
\begin{align*}
        R_t \le \begin{cases}
        \frac{1}{2} + \ln\frac{ 2}{ \ln\ln c} + \ln\ln\left(c\sqrt{1+\varprocess}\right) + 2\ln\ln\ln(c\sqrt{1+\varprocess}), & \text{ if }\frac{|S_t|}{\varprocess}\le \frac{1}{\sqrt{1+\varprocess}} \\
        \frac{|S_t|}{\sqrt{V_t}}\sqrt{1+\frac{1}{V_t}} -
        \frac{1}{2} + \ln\frac{ 2}{ \ln\ln c} \\
        \qquad\qquad + \ln\ln\left(c\sqrt{1+\varprocess}\right) + 2\ln\ln\ln(c\sqrt{1+\varprocess}), & \text{ if } \frac{1}{\sqrt{1+\varprocess}}<\frac{|S_t|}{\varprocess} \le 1. 
    \end{cases}
\end{align*}
Using~\eqref{eq:snlil} in the above bound, on $\calE_0$, $R_t \le \ln\ln \varprocess (1+o(1))$, eventually. 
\end{proof}

Shafer and Vovk, in their book \citep[Chapter~5]{shafer2005probability}, present a game-theoretic version of the LIL that is closely related to our Theorem~\ref{th:lilasregretV}. Their proof sets up a game between three players: the forecaster, the skeptic, and reality, and shows the existence of a betting strategy (a non-negative supermartingale) that makes the skeptic infinitely rich if an event of the form $\mathcal{E}_0$ fails. Thus, whenever the skeptic’s wealth remains finite, the LIL holds. \citet[Theorem 5.2]{shafer2005probability} prove an exact LIL statement, but under the assumption that the observations are bounded. They also prove just the upper LIL in \citet[Theorem 5.1]{shafer2005probability} under slightly weaker assumptions. However, their bounds only hold if the data are in a bounded range, and they do not derive explicit regret bounds. Our results extend in complementary directions: we obtain explicit path-wise regret bounds that hold for all $t$ for unbounded data. Further, we establish an upper LIL for self-normalized processes. Our analysis exhibits an explicit supermartingale (wealth process) that witnesses our version of the LIL. We refer the reader to Appendix~\ref{app:bettinggame} for a detailed description of the betting game involving the three players.

\begin{remark}
Despite the apparent focus of our bounds above on sub-Gaussian processes (Definition~\ref{def:subgaussian}), note that they also apply when the data is drawn from a mean-zero distribution with only a finite second moment. Such distributions can still be treated as sub-Gaussian, albeit with a different variance proxy (see Point~\ref{heavy} in the examples in Appendix~\ref{app:subGaussian}). Moreover, recall that the finiteness of the second moment is not only sufficient but also necessary for LIL to hold.
\end{remark}

\subsection{Proof of Theorem~\ref{th:lilregretV}}\label{sec:proof_thlilregretV}
In this section, we give a proof of Theorem~\ref{th:lilregretV}. It relies on certain auxiliary lemmas and propositions. While we state these results here, their proofs are given in Appendix~\ref{app:proofs_lilregret}.

Recall from Section~\ref{sec:setup} that $R_t:= f_t(\uncopt) - \ln W_t$, and that for $t$ such that $\varprocess = 0$, we also have $S_t = 0$, and hence for such $t$, for all $\eta \in \R$, $f_t(\eta) = 0$, and $\ln W_t = 0$. Thus, the inequalities in the theorem hold trivially in this case. Henceforth, we focus on $t$ so that $\varprocess > 0$. 

\subsubsection{Step 1: proving regret bound along all data sequences}\label{sec:step1}
The first two cases in~\eqref{eq:pathwiseregretrobbins} correspond to $t$ such that the unconstrained optimizer $\uncopt$ is in $[-1,1]$. To get the bound on $R_t$ in~\eqref{eq:pathwiseregretrobbins} in these cases, we show that the mixture wealth $W_t$, which is an integration of a non-negative function over $[-1,1]$ weighted by a prior $\pi$, is at least the value of the integration over a subset that contains the maximizer $\uncopt$. This subset is chosen large enough so that it contains a significant fraction of the prior mass. Further, it is also small enough so that the value of the function being integrated remains almost constant (and almost equal to the maximum value) in this range.

Lemma~\ref{lem:regretinterior} below chooses this subset to be an interval of length $r$, centered at the optimizer $\uncopt$.

\begin{lemma}\label{lem:regretinterior}\label{lem:master-explicit} 
For every $t\ge 1$ such that $|\uncopt| \le 1$, for every $r\in(0,1)$,
\begin{equation*}
R_t \le \frac{\varprocess r^2}{2} - \ln\lrp{\int\limits_{\substack{\lrset{\eta: |\eta-\uncopt|\le r}\\ \cap [-1,1]}} \pi(\eta) d\eta}.
\end{equation*}
\end{lemma}

For $r\in (0,1)$, for all feasible $\eta$ in the chosen interval $\lrset{\eta: |\eta - \uncopt| \le r} \cap [-1,1]$, the function $e^{f_t(\cdot)}$ can be expressed as the maximum value adjusted by a small cost. Lemma~\ref{lem:regretinterior} bounds this cost to arrive at a bound on $R_t$ in terms of $\varprocess$, $r$, and the mass that the prior $\pi$ puts on the chosen interval. 

Next, Lemma~\ref{lem:window-explicit} below, lower bounds the mass that the prior $\pi$ places on the chosen interval. It uses the shape of $\pi$, and properties like convexity and radial monotonicity of $\pi$, to arrive at the bounds.

\begin{lemma}\label{lem:window-explicit}
For $|\uncopt| < 1$ and any $r\in(0,1)$:
\begin{align*}
    \int\limits_{\lrset{\eta: |\eta-\uncopt|\le r}\cap [-1,1]}\pi(\eta) d\eta
\ge 
\begin{cases} \frac{\ln\ln c}{2{\ln(c/r)}\lrp{\ln\ln(c/r)}^2}, \quad &\text{if } |\uncopt|\le r,\\
\frac{r\ln\ln c}{2|\uncopt|\lrp{\ln(c/r)}\lrp{\ln\ln(c/r)}^2}, \quad&\text{if }|\uncopt|>r.
\end{cases}
\end{align*}
\end{lemma}

The bound in~\eqref{eq:pathwiseregretrobbins} in the first two cases follows by plugging that from Lemma~\ref{lem:window-explicit} into Lemma~\ref{lem:regretinterior}, setting $r = \frac{1}{\sqrt{1+\varprocess}}$, and using $\frac{\varprocess r^2}{2} \le \frac{1}{2}$. 

Coming to the third branch in~\eqref{eq:pathwiseregretrobbins}, we again express the function $e^{f_t(\cdot)}$ in terms of $e^{f_t(\uncopt)}$ with the adjustment terms.  However, since the optimizer $\uncopt$ lies outside $[-1,1]$ in this case, we choose the radius $r$ large enough, so that the interval of length $r$ centered at $\uncopt$ still has enough intersection with $[-1,1]$ (recall, $\pi$ is only supported on $[-1,1]$). Since the feasible $\eta$ in this intersection are far from $\uncopt$, the cost incurred in relating $e^{f_t(\cdot)}$ to $e^{f_t(\uncopt)}$ is relatively high in this case compared to the previous two cases. Lemma~\ref{lem:regretboundary} below formalizes this and proves the inequality in the third case in~\eqref{eq:pathwiseregretrobbins}. First term in the bound corresponds to the cost of relating $e^{f_t(\cdot)}$ with $e^{f_t(\uncopt)}$ for $\eta\in[-1,1] \cap \{\eta: |\eta - \uncopt| \le r\}$. The second term is the mass that the prior $\pi$ puts on this set.

\begin{lemma}\label{lem:regretboundary}
For $t \ge 1$ such that  $|\uncopt|>1$, for $\rho\in(0,1)$,
\[
R_t \le  \frac{\varprocess}{2} \left(\frac{|S_t|}{V_t}-1+\rho\right)^2 - \ln\left(\frac{\rho\ln\ln c}{\lrp{\ln\tfrac{c}{1-\rho}}\lrp{  \ln\ln\tfrac{c}{1-\rho}}^2}\right).
\]
\end{lemma}

It now remains to show that the bound in~\eqref{eq:pathwiseregretrobbins} is not too large on $\calE_\alpha$. Moreover, $\calE_\alpha$ holds along most of the data sequences when the data satisfies appropriate stochasticity assumptions. We show these in the following section.

\subsubsection{Step 2: conditional regret bound on $\calE_\alpha$}\label{sec:step2}
Recall $\calE_\alpha = \{ \sup_t \ln W_t \le \ln(1/\alpha) \}$. Observe that to show~\eqref{eq:lilregretonealpha}, it suffices to show that the bounds in the middle and last branches of~\eqref{eq:pathwiseregretrobbins} are of the form in~\eqref{eq:lilregretonealpha} on $\calE_\alpha$ (it holds for the first branch with $C = 4$).

We first show~\eqref{eq:lilregretonealpha} for the case where $\frac{|S_t|}{\varprocess} > 1$. As an intermediate step, Proposition~\ref{prop:intermediatebound} below first shows that on $\calE_\alpha$, if $\varprocess$ is larger than a threshold that depends only on $\alpha$, say $V_\alpha$, the third branch cannot occur. Using this observation, it further bounds $R_t$ on the third branch in terms of this threshold $V_\alpha$,  $\varprocess$, and a parameter $\rho$. 

\begin{proposition}\label{prop:intermediatebound}
Let $\rho \in (0,1/4)$. The following hold on $\calE_\alpha$. 
\begin{enumerate}
\item There exists a threshold 
\[ V_\alpha := \frac{2}{1-\rho^2} \lrp{ \ln\frac{1}{\alpha}- \ln\frac{\rho\ln\ln c}{2} + \ln\ln\frac{c}{1-\rho}  + 2\ln\ln\ln\frac{c}{1-\rho}} \numberthis \label{eq:Tdelta} \]
such that for all $t$ such that $\varprocess \ge V_\alpha$, we have $|\uncopt| \le 1$. 
\item For all $t \ge 1$ such that $|\eta^*_t| \ge 1$ (or equivalently, $\frac{|S_t|}{V_t} \ge 1$), 
\begin{align}\label{eq:regretboundonedelta}
    R_t \le \frac{(1+\rho)^2 V^2_\alpha}{8\varprocess} + \frac{\rho V_\alpha}{2} (1+2\rho) - \ln \lrp{ \tfrac{\rho \ln\ln c}{\lrp{\ln{\tfrac{c}{1-\rho}}}\lrp{\ln\ln\tfrac{c}{1-\rho}}^2} }.
\end{align}
\end{enumerate}
\end{proposition}

   With the above proposition, we get that~\eqref{eq:lilregretonealpha} holds in this case, with
   \[ C = \frac{3}{(1-\rho)^2}\lrp{1\vee \lrset{\ln ({\rho \ln\ln c}) - \ln\ln \tfrac{c}{1-\rho} - 2\ln\ln\ln\tfrac{c}{1-\rho} - \ln 2}}^2. \]
Now, it only remains to prove~\eqref{eq:lilregretonealpha} for $t\ge 1$ such that $\tfrac{1}{\sqrt{1+\varprocess}} < \tfrac{ |S_t|}{\varprocess} \le 1$. In this case, since $\ln(x) \le \tfrac{x^2}{4}$ for $x\ge 0$, we have from~\eqref{eq:pathwiseregretrobbins} that
    \begin{align*}
        R_t &= \frac{S^2_t}{2 \varprocess} - \ln W_t\\
        &\le \lrp{\frac{1}{2} + \ln\frac{2}{ \ln\ln c}} + \ln\ln\lrp{c \sqrt{1+\varprocess} } + \frac{S^2_t}{4 \varprocess}\lrp{1+\frac{1}{\varprocess}} + 2\ln\ln\ln(c\sqrt{1+\varprocess}). \numberthis\label{eq:intermediate1}
    \end{align*}
    Rearranging the above inequality, using that ${|S_t|} \le {\varprocess}$, and that on $\calE_\alpha$, $\sup_t\ln W_t \le \ln\frac{1}{\alpha}$, we get the following on $\calE_\alpha$ in this case:
    \[  \frac{S^2_t}{4 \varprocess } \le  \ln\frac{1}{\alpha} + \lrp{\frac{1}{2} + \ln\frac{2}{ \ln\ln c}} + \ln\ln(c \sqrt{1+\varprocess} ) + \frac{1}{4} + 2 \ln\ln\ln(c\sqrt{1+\varprocess}). \]
    Substituting the above back in~\eqref{eq:intermediate1}, 
    we get that on $\calE_\alpha$, the following holds for all $t\ge 1$ such that $\frac{1}{\sqrt{1+\varprocess}} < \frac{|S_t|}{\varprocess} \le 1$:
    \begin{align*} 
        R_t &\le  \ln\frac{1}{\alpha} + 2\lrp{\frac{1}{2} + \ln\frac{2}{\ln\ln c}} + 2\ln\ln(c  \sqrt{1+\varprocess} )+\frac{1}{2} + 4 \ln\ln\ln(c\sqrt{1+\varprocess}),
    \end{align*}
    proving~\eqref{eq:lilregretonealpha} in this case, with $C = 7$.  

    \subsubsection{Step 3: stochastic data}\label{sec:step3}
    Finally, if the data are drawn from a distribution $P$ such that $(S_t)_{t\ge 1}$ is a sub-Gaussian process with variance proxy $(\varprocess)_{t\ge 1}$, then $(e^{f_t(\eta)})_{t\ge 1}$ is a non-negative super-martingale for every $\eta\in\R$  (Definition~\ref{def:subgaussian}). Thus, $(W_t)_{t\ge 1}$ is also a non-negative super-martingale. The last statement then follows from Ville's inequality \citep{ville1939etude}. $\Box$

\section{Trade-off between regret and growth rate}
In this section, we point out an apparently inherent trade-off between the growth rate of mixture martingales and their rate of accumulating regret: asymptotic optimality in one precludes asymptotic optimality in the other.

When the data is stochastic (i.i.d. from a fixed distribution $Q$), the lower limit $\liminf\nolimits_{t\to\infty} \frac{1}{V_t} \ln W_t$ is typically $Q$-a.s. a constant, and in that case, that constant is termed as the asymptotic growth rate of $W_t$ against $Q$ \citep[Chapter 7]{ramdas2024hypothesis}. It quantifies the rate at which the wealth grows exponentially under $Q$. In this section, we compare the two wealth processes defined in~\eqref{eq:mixV} and~\eqref{eq:robbinsmixwealth} at the level of their growth rate and confidence sequence (CS) widths. In contrast to the rest of the paper, this section is purely stochastic. We assume that the data is drawn i.i.d. from a fixed distribution $Q$ that satisfies $\frac{|S_t|}{V_t} \to m$ as $t\to\infty$, for some $m \in R_+$. 

For the Gaussian-mixture from Section~\ref{sec:warmuplogTregret}, using the regret bound from~\eqref{eq:boundrtgaussianmix}, we have
\[ \liminf\limits_{t\rightarrow\infty} \frac{\ln W_t}{V_t} = \liminf\limits_{t\rightarrow \infty} \lrp{\frac{S^2_t}{2V^2_t} - \frac{R_t}{V_t}} \ge \liminf\limits_{t\rightarrow\infty} \frac{S^2_t}{2V^2_t} = \frac{m^2}{2}. \]
However, as we saw in Section~\ref{sec:warmuplogTregret}, this wealth process has a regret that is $O(\ln V_t)$, and does not achieve $O(\ln\ln V_t)$ regret.

Let us now consider the mixture wealth process defined in~\eqref{eq:robbinsmixwealth} using Robbins' prior. From Theorem~\ref{th:lilregretV}, we see that 
\[ \limsup\limits_{t\rightarrow \infty}~ \frac{R_t}{V_t} \le \begin{cases}
\limsup\limits_{t\rightarrow \infty}~ \frac{1}{2}\lrp{ \frac{|S_t|}{V_t} - 1 + \rho }^2, & \text{ if } 1 < \frac{|S_t|}{V_t}\\
    0, & \text{ otherwise}.
\end{cases} \]
Since Robbins' mixture puts a large prior mass at bets close to $0$, it has relatively small mass at larger bets. Therefore, when the data is such that the optimal bet $\frac{|S_t|}{V_t}$ is far from $0$ (in particular $> 1$), it lacks sufficient weight to compete with the best-in-hindsight, and hence, suffers large regret in those cases. While the bound above gives a linear \emph{upper bound} on the regret in these cases, we numerically observe that the \emph{regret is indeed linear} in such cases. This linear regret, in fact, affects the growth rate of the wealth process, when $m$ is far from $0$. To see this, consider $m > 1$. In this case, the growth rate for the Robbins' mixture wealth is smaller than $m^2/2$ due to the linear regret suffered in this case. It is $m^2/2$ when $m \in [0,1]$. 

Next, it is well-known that for a fixed error probability $\alpha \in (0,1)$, the CS derived from the Gaussian mixture wealth~\eqref{eq:mixV} are tighter early on (for small $t$) than those obtained using Robbins’ mixture. However, from~\eqref{eq:cs}, we can see that they are asymptotically looser with an asymptotic width of $O(\ln V_t)$, while the latter has $O(\ln\ln V_t)$ width. The relative tightness of the former (for small $t$) can be explained from the faster growth rate of the corresponding wealth process, as discussed next. 

Recall that one way to construct a $(1-\alpha)$-CS for mean is to run, in parallel, a collection of one-sided sequential tests of the null hypothesis that the mean is $m_0$, for each $m_0\in\R$. At time $t$, the $(1-\alpha)$-CS is the set of all values $m_0$ whose corresponding test has not yet stopped (and hence, has not rejected $m_0$). For simplicity, consider testing the mean of a Gaussian distribution with known unit variance. To test mean being $m_0$, we may use either of the mixture wealth processes in~\eqref{eq:mixV} or~\eqref{eq:robbinsmixwealth}, with $S_t = \sum_{i\le t} (X_i - m_0)$, and $V_t = t$. Both wealth processes are non-negative martingales under the null, and therefore yield valid one-sided sequential tests with the stopping time given by 
$$\tau_{m_0}:= \min\lrset{t : W_t \ge \frac{1}{\alpha}}.$$

From above, one can deduce that for fixed $m_0$, the stopping time (to reject $m_0$) will be smaller if the wealth $W_t$ grows faster. Since the growth rate for Robbins' mixture wealth from~\eqref{eq:robbinsmixwealth} is smaller for $m_0$ that are ``far'' from the true mean, the corresponding test rejects $m_0$ later than that using Gaussian mixture from~\eqref{eq:mixV}. However, once the ``far-off'' $m_0$ are excluded, the Robbins' wealth suffers a smaller regret, and hence has a faster growth rate (in non-asymptotic sense) than the Gaussian mixture. Thus, it rejects the nearby $m_0$ faster, leading to narrower CS eventually. In fact, we don't see the CS derived using Robbins' mixture wealth process perform poorly in practice. This is because the mean values very far from the true mean are eliminated from both the CSs early on (since we often pick $\alpha=0.05$), since the signal from them is high, but the effect would be more visible for small $\alpha$. 

Thus, there appears to be a trade-off between the asymptotic growth rate of the mixture wealth process and the regret incurred. To achieve the smaller $O(\ln\ln V_t)$ regret when the mean shift is small, the mixture in~\eqref{eq:robbinsmixwealth} pays in its asymptotic growth rate against larger mean shifts. To the best of our knowledge, these observations have not been pointed out in the literature.

\section{Conclusions}\label{sec:conc}
In this work, we proved deterministic regret bounds for sub-Gaussian mixture processes of the form $W_t$, which mix $\exp\{\eta S_t - \eta^2 \varprocess / 2\}$ over $\eta\in\R$. Our regret bounds hold for all sequences of data. However, without any assumption on the data sequence, they can be large (for example, linear in $t$ or $V_t$). We then show that for certain sequences of data, specifically along which the mixture $W_t$ remains uniformly bounded, the regret for the mixture using Robbins' heavy-near-0 prior is $O(\ln\ln \varprocess)$ for every $t$. Furthermore, if the data are assumed to be stochastic, then the set of such sequences is large, and hence $O(\ln \ln \varprocess)$ holds along many sequences of observations. In fact, we show that regret is $O(\ln\ln \varprocess)$ eventually on a set of measure $1$. 
    
We strongly believe that the deterministic regret bounds of the form we prove in this work for sub-Gaussian wealth processes from the previous paragraph, also hold for a broader class of (sub-$\psi$) wealth processes obtained by mixing $\exp\{\eta S_t - \psi(\eta) \varprocess\}$ over $\eta$ for a non-negative, convex function $\psi$. Moreover, such mixtures are (super)martingales under a much broader (sub-$\psi$) class of probability measures (see, \cite{howard2020time} for a definition). When the data is assumed to be stochastic from such a distribution, we expect to get similar conditional regret bounds on Ville's events. Our work, thus, gives a path towards writing down (conditional) regret bounds for unbounded data, which may be shown to be small for stochastic settings. We, in fact, conjecture that:
\begin{quote}
Behind all concentration results that are derived by constructing an appropriate non-negative (super)martingale, followed by an application of Ville's inequality, there exists a conditional regret bound, that is, a bound on the regret of the chosen martingale that is small on the Ville event. 
\end{quote}

Finally, in game-theoretic statistics, non-negative (super)martingales, or equivalently, wealth processes in a (conservative) fair game, play a central role, together with Ville’s inequality. At their core, these are probabilistic objects. In parallel, the field of adversarial online learning is built around deterministic, path-wise regret inequalities. There appear to be strong connections between these two worlds. For example, Ville, in his thesis \citep{ville1939etude}, proved a striking deterministic statement about martingales: given a probability measure $P$, for every event $A$ of measure zero, there exists a nonnegative $P$-martingale $M^P$ that becomes unbounded on $A$. Ville’s martingale theorem thus serves as one bridge between the probabilistic and deterministic viewpoints. Related connections have been explored in subsequent works,  \citep{rakhlin2017equivalence, acciaio2013trajectorial, beiglbock2014martingale, beiglbock2015pathwise, gushchin2014pathwise,shafer2005probability,clerico2025confidence}, which derive probabilistic results from certain deterministic inequalities. However, a complete and unified understanding of the relationship between the two perspectives is still missing in the literature. We view our contributions as another step toward unifying these two paradigms. 

\section*{Acknowledgements}
{SA acknowledges the generous support from the Pratiksha Trust, Bangalore, through the Young Investigator Award, and the DST INSPIRE Faculty Grant IFA24-ENG-389. AR acknowledges funding from NSF grant 2310718 and a Sloan fellowship.}

\bibliography{BibTex}

\begin{thebibliography}{34}
\providecommand{\natexlab}[1]{#1}
\providecommand{\url}[1]{\texttt{#1}}
\expandafter\ifx\csname urlstyle\endcsname\relax
  \providecommand{\doi}[1]{doi: #1}\else
  \providecommand{\doi}{doi: \begingroup \urlstyle{rm}\Url}\fi

\bibitem[Acciaio et~al.(2013)Acciaio, Beiglb{\"o}ck, Penkner, Schachermayer,
  and Temme]{acciaio2013trajectorial}
B~Acciaio, M~Beiglb{\"o}ck, F~Penkner, W~Schachermayer, and J~Temme.
\newblock A trajectorial interpretation of {Doob's} martingale inequalities.
\newblock \emph{The Annals of Applied Probability}, pages 1494--1505, 2013.

\bibitem[Agrawal et~al.(2021)Agrawal, Koolen, and Juneja]{agrawal2021optimal}
Shubhada Agrawal, Wouter~M Koolen, and Sandeep Juneja.
\newblock Optimal best-arm identification methods for tail-risk measures.
\newblock \emph{Advances in Neural Information Processing Systems},
  34:\penalty0 25578--25590, 2021.

\bibitem[Beiglb{\"o}ck and Nutz(2014)]{beiglbock2014martingale}
Mathias Beiglb{\"o}ck and Marcel Nutz.
\newblock Martingale inequalities and deterministic counterparts.
\newblock \emph{Electronic Journal of Probability}, 19:\penalty0 1--15, 2014.

\bibitem[Beiglb{\"o}ck and Siorpaes(2015)]{beiglbock2015pathwise}
Mathias Beiglb{\"o}ck and Pietro Siorpaes.
\newblock Pathwise versions of the {B}urkholder-{D}avis-{G}undy inequality.
\newblock \emph{Bernoulli}, pages 360--373, 2015.

\bibitem[Clerico et~al.(2025)Clerico, Flynn, and Neu]{clerico2025confidence}
Eugenio Clerico, Hamish Flynn, and Gergely Neu.
\newblock Confidence sequences for generalized linear models via regret
  analysis.
\newblock \emph{arXiv preprint arXiv:2504.16555}, 2025.

\bibitem[Cover(1991)]{cover1991universal}
Thomas~M Cover.
\newblock Universal portfolios.
\newblock \emph{Mathematical finance}, 1\penalty0 (1):\penalty0 1--29, 1991.

\bibitem[Cover and Ordentlich(2002)]{cover2002universal}
Thomas~M Cover and Erik Ordentlich.
\newblock Universal portfolios with side information.
\newblock \emph{IEEE Transactions on Information Theory}, 42\penalty0
  (2):\penalty0 348--363, 2002.

\bibitem[Cutkosky and Orabona(2018)]{cutkosky2018black}
Ashok Cutkosky and Francesco Orabona.
\newblock Black-box reductions for parameter-free online learning in {B}anach
  spaces.
\newblock In \emph{Conference On Learning Theory}, pages 1493--1529. PMLR,
  2018.

\bibitem[Cutkosky and Sarlos(2019)]{cutkosky2019matrix}
Ashok Cutkosky and Tamas Sarlos.
\newblock Matrix-free preconditioning in online learning.
\newblock In \emph{International Conference on Machine Learning}, pages
  1455--1464. PMLR, 2019.

\bibitem[Darling and Robbins(1967{\natexlab{a}})]{darling1967confidence}
Donald~A Darling and Herbert Robbins.
\newblock Confidence sequences for mean, variance, and median.
\newblock \emph{Proceedings of the National Academy of Sciences}, 58\penalty0
  (1):\penalty0 66--68, 1967{\natexlab{a}}.

\bibitem[Darling and Robbins(1967{\natexlab{b}})]{darling1967iterated}
Donald~A Darling and Herbert Robbins.
\newblock Iterated logarithm inequalities.
\newblock \emph{Proceedings of the National Academy of Sciences}, 57\penalty0
  (5):\penalty0 1188--1192, 1967{\natexlab{b}}.

\bibitem[de~la Pena(2004)]{victor2004self}
Victor~H de~la Pena.
\newblock Self-normalized processes: Exponential inequalities, moment bounds
  and iterated logarithm laws.
\newblock \emph{The Annals of Probability}, 32\penalty0 (3A):\penalty0
  1902--1933, 2004.

\bibitem[de~la Pena et~al.(2009)de~la Pena, Lai, and Shao]{de2009self}
Victor~H de~la Pena, Tze~Leung Lai, and Qi-Man Shao.
\newblock \emph{Self-normalized processes: Limit theory and Statistical
  Applications}.
\newblock Springer, 2009.

\bibitem[Gr{\"u}nwald(2007)]{grunwald2007minimum}
Peter~D Gr{\"u}nwald.
\newblock \emph{The minimum description length principle}.
\newblock MIT press, 2007.

\bibitem[Gushchin(2014)]{gushchin2014pathwise}
AA~Gushchin.
\newblock On pathwise counterparts of {D}oob’s maximal inequalities.
\newblock \emph{Proceedings of the Steklov Institute of Mathematics},
  287\penalty0 (1):\penalty0 118--121, 2014.

\bibitem[Howard et~al.(2020)Howard, Ramdas, McAuliffe, and
  Sekhon]{howard2020time}
Steven~R Howard, Aaditya Ramdas, Jon McAuliffe, and Jasjeet Sekhon.
\newblock Time-uniform {C}hernoff bounds via nonnegative supermartingales.
\newblock \emph{Probability Surveys}, 2020.

\bibitem[Howard et~al.(2021)Howard, Ramdas, Mcauliffe, and
  Sekhon]{howard2021time}
Steven~R Howard, Aaditya Ramdas, Jon Mcauliffe, and Jasjeet Sekhon.
\newblock Time-uniform, nonparametric, nonasymptotic confidence sequences.
\newblock \emph{The Annals of Statistics}, 49\penalty0 (2):\penalty0
  1055--1080, 2021.

\bibitem[Kaufmann and Koolen(2021)]{kaufmann_mm_21}
Emilie Kaufmann and Wouter~M. Koolen.
\newblock Mixture martingales revisited with applications to sequential tests
  and confidence intervals.
\newblock \emph{Journal of Machine Learning Research}, 22\penalty0
  (246):\penalty0 1--44, 2021.

\bibitem[Koolen and Van~Erven(2015)]{koolen2015second}
Wouter~M Koolen and Tim Van~Erven.
\newblock Second-order quantile methods for experts and combinatorial games.
\newblock In \emph{Conference on Learning Theory}, pages 1155--1175. PMLR,
  2015.

\bibitem[Orabona(2019)]{orabona2019modern}
Francesco Orabona.
\newblock A modern introduction to online learning.
\newblock \emph{arXiv preprint arXiv:1912.13213}, 2019.

\bibitem[Orabona and Jun(2023)]{orabona2023tight}
Francesco Orabona and Kwang-Sung Jun.
\newblock Tight concentrations and confidence sequences from the regret of
  universal portfolio.
\newblock \emph{IEEE Transactions on Information Theory}, 2023.

\bibitem[Orabona and P{\'a}l(2016)]{orabona2016coin}
Francesco Orabona and D{\'a}vid P{\'a}l.
\newblock Coin betting and parameter-free online learning.
\newblock \emph{Advances in Neural Information Processing Systems}, 29, 2016.

\bibitem[Rakhlin and Sridharan(2017)]{rakhlin2017equivalence}
Alexander Rakhlin and Karthik Sridharan.
\newblock On equivalence of martingale tail bounds and deterministic regret
  inequalities.
\newblock In \emph{Conference on Learning Theory}, pages 1704--1722. PMLR,
  2017.

\bibitem[Ramdas and Wang(2025)]{ramdas2024hypothesis}
Aaditya Ramdas and Ruodu Wang.
\newblock \emph{Hypothesis testing with e-values}.
\newblock Foundations and Trends in Statistics, 2025.

\bibitem[Ramdas et~al.(2023)Ramdas, Gr{\"u}nwald, Vovk, and
  Shafer]{ramdas2023game}
Aaditya Ramdas, Peter Gr{\"u}nwald, Vladimir Vovk, and Glenn Shafer.
\newblock Game-theoretic statistics and safe anytime-valid inference.
\newblock \emph{Statistical Science}, 38\penalty0 (4):\penalty0 576--601, 2023.

\bibitem[Rissanen(2002)]{rissanen2002fisher}
Jorma~J Rissanen.
\newblock Fisher information and stochastic complexity.
\newblock \emph{IEEE Transactions on Information Theory}, 42\penalty0
  (1):\penalty0 40--47, 2002.

\bibitem[Robbins and Siegmund(1968)]{robbins1968iterated}
Herbert Robbins and David Siegmund.
\newblock Iterated logarithm inequalities and related statistical procedures.
\newblock \emph{Mathematics of the Decision Sciences}, 2:\penalty0 267--279,
  1968.

\bibitem[Robbins and Siegmund(1970)]{robbins1970boundary}
Herbert Robbins and David Siegmund.
\newblock Boundary crossing probabilities for the wiener process and sample
  sums.
\newblock \emph{The Annals of Mathematical Statistics}, pages 1410--1429, 1970.

\bibitem[Ruf et~al.(2023)Ruf, Larsson, Koolen, and Ramdas]{ruf2023composite}
Johannes Ruf, Martin Larsson, Wouter~M Koolen, and Aaditya Ramdas.
\newblock A composite generalization of ville’s martingale theorem using
  e-processes.
\newblock \emph{Electronic Journal of Probability}, 28:\penalty0 1--21, 2023.

\bibitem[Shafer(2021)]{shafer2021testing}
Glenn Shafer.
\newblock Testing by betting: A strategy for statistical and scientific
  communication.
\newblock \emph{Journal of the Royal Statistical Society Series A: Statistics
  in Society}, 184\penalty0 (2):\penalty0 407--431, 2021.

\bibitem[Shafer and Vovk(2005)]{shafer2005probability}
Glenn Shafer and Vladimir Vovk.
\newblock \emph{Probability and finance: it's only a game!}, volume 491.
\newblock John Wiley \& Sons, 2005.

\bibitem[Shafer and Vovk(2019)]{shafer2019game}
Glenn Shafer and Vladimir Vovk.
\newblock \emph{Game-theoretic foundations for probability and finance}.
\newblock John Wiley \& Sons, 2019.

\bibitem[Ville(1939)]{ville1939etude}
Jean Ville.
\newblock \emph{Etude critique de la notion de collectif}.
\newblock Gauthier-Villars Paris, 1939.

\bibitem[Waudby-Smith and Ramdas(2024)]{waudby2024estimating}
Ian Waudby-Smith and Aaditya Ramdas.
\newblock Estimating means of bounded random variables by betting.
\newblock \emph{Journal of the Royal Statistical Society Series B: Statistical
  Methodology}, 86\penalty0 (1):\penalty0 1--27, 2024.

\end{thebibliography}

\appendix

\section{Sub-Gaussian Process with Variance Proxy}\label{app:subGaussian}
In this section, we recall some examples of sub-Gaussian processes $(S_t)_t$ with the corresponding variance proxy $(\varprocess)_t$ to demonstrate that Definition~\ref{def:subgaussian} is quite general. 

\begin{enumerate}
\item\label{gaussian} First, if $X_i$ are drawn iid from the standard Gaussian distribution (with mean $0$ and variance $1$), then $(S_t)_{t\ge 1}$ is a 1-sub-Gaussian process with variance proxy $\varprocess = t$. This is easy to see since for all $\eta\in\R$, $\exp{\eta S_t - \eta^2 t / 2}$ is a non-negative martingale, satisfying the conditions in Definition~\ref{def:subgaussian}.

\item\label{subG} For $\sigma > 0$, if $X_i$ are drawn iid from a centered (mean zero) $\sigma^2$-sub-Gaussian distribution $P$ (that satisfies $\mathbb{E}_P[\exp(\eta X - \sigma^2\eta^2/2)] \le 1$), then $(S_t)_{t\ge 1}$ is sub-Gaussian with $V_t = \sigma^2 t$.
    
\item\label{sym} If $X_t$ are drawn iid from a distribution $P$ that is symmetric around $0$, then $(S_t)_{t\ge 1}$ is sub-Gaussian with $\varprocess = \sum\nolimits_{i=1}^t X^2_i$. Notably, in this case, $P$ may not even have a finite first moment (for instance, Cauchy distribution). 
    
\item\label{heavy} If $X_t$ are drawn iid from a distribution $P$ with mean $0$ and finite variance (i.e., $\mathbb{E}_P[X^2] < \infty$), then $(S_t)_{t\ge 1}$ is sub-Gaussian with $V_t = \frac{1}{3}\sum\nolimits_{i=1}^t (X^2_i + 2 \mathbb{E}[X^2_i])$.
\end{enumerate}

The examples above show that the condition encompasses both light-tailed (cases~\ref{gaussian} and~\ref{subG}) and heavy-tailed (cases~\ref{sym} and~\ref{heavy}) settings. Moreover, the condition of $X_i$ being iid can be greatly relaxed. We refer the reader to \citet[Table S3]{howard2021time} and the discussions therein for more details. 

\section{Proofs from Section~\ref{sec:warmuplogTregret}}\label{app:proof_warmuplogTregret}
\begin{proof}[Proof of Theorem~\ref{th:ville_gaussianmixture}]
We first prove the expression for $R_t$. To this end, observe from~\eqref{eq:mixV}, that 
\[
-\ln W_t = \frac{1}{2}\ln(1+\sigma_0^2 V_t) - \frac{\sigma_0^2 S_t^2}{2(1+\sigma_0^2 V_t)},\numberthis\label{eq:z_tgaussianmix}
\]
which equals $0$ when $V_t = 0$. Add $L_t^* = S_t^2/(2V_t)$ (or $0$ when $V_t=0$) to obtain
\[
R_t = \frac{1}{2}\ln(1+\sigma_0^2 V_t) + \frac{S_t^2}{2V_t} - \frac{\sigma_0^2 S_t^2}{2(1+\sigma_0^2 V_t)}
= \frac{1}{2}\ln(1+\sigma_0^2 V_t) + \frac{S_t^2}{2V_t(1+\sigma_0^2 V_t)},
\]
proving the expression for $R_t$ from the theorem statement.

Next, rearrange~\eqref{eq:z_tgaussianmix} to get 
\[
\frac{S_t^2}{1+\sigma_0^2 V_t}
 =  \frac{2}{\sigma_0^2}\left(\ln W_t + \tfrac12\ln(1+\sigma_0^2 V_t)\right).
\]
Plugging this into the expression for $R_t$ from above, we get
\[
R_t = \tfrac12\ln(1+\sigma_0^2 V_t) + \frac{1}{\sigma_0^2 V_t}\left(\ln W_t + \tfrac12\ln(1+\sigma_0^2 V_t)\right).
\]
On $\calE_\alpha$, since $\ln W_t\le \ln(1/\alpha)$ for all $t$, we further get
\[ R_t \le \tfrac12\ln(1+\sigma_0^2 V_t) + \frac{1}{\sigma_0^2 V_t}\left(\ln \frac{1}{\alpha} + \tfrac12\ln(1+\sigma_0^2 V_t)\right). \]
When $V_t\ge v_0>0$, replacing $V_t$ by $v_0$ in the denominators of the above expression, we get the desired bound.

Finally, if $X_i$ are iid from $P$, and $(S_t)_{t\ge 1}$ is sub-Gaussian with variance $(V_t)_{t\ge 1}$, then for  $\eta\in\R$, the process $(e^{f_t(\eta)})_{t\ge 1}$ is a non-negative super-martingale (Definition~\ref{def:subgaussian}). Hence, $(W_t)_{t\ge 1}$ is also a non-negative super-martingale. Ville's inequality \citep{ville1939etude} then implies that $P[\calE_\alpha] \ge 1-\alpha$.
\end{proof}

\section{Proofs from Section~\ref{sec:lilregret}}\label{app:proofs_lilregret}
\begin{proof}[Proof of Lemma~\ref{lem:regretinterior}]
Recall that $\uncopt = \frac{S_t}{\varprocess}$ and $L^*_t = f_t(\eta^*_t) = \frac{S^2_t}{2V_t}$. On algebraic manipulation,  we get $f_t(\eta)=L_t^\star-\frac{\varprocess}{2}(\eta-\uncopt)^2$, which implies
\begin{align*}
    W_t = \int\limits_{-1}^1 e^{f_t(\eta)} \pi(\eta) d\eta 
    &= e^{L^\star_t} \int\limits_{-1}^1 e^{- \frac{\varprocess}{2}(\eta - \uncopt)^2} \pi(\eta) d\eta\ge e^{L^\star_t - \frac{1}{2}\varprocess r^2} \int\limits_{\substack{\lrset{\eta: |\eta-\uncopt| \le r} \\ \cap [-1,1]}}  \pi(\eta) d\eta.
\end{align*}
Upon taking logs and rearranging, we get the desired inequality. 
\end{proof}

\begin{proof}[Proof of Lemma~\ref{lem:window-explicit}]
For $y\in (0,1]$ and $c > 6.6e$ chosen as in $\pi(\cdot)$, let
$$\phi(y):= \frac{1}{y\lrp{\ln\frac{c}{y}}\lrp{\ln\ln\frac{c}{y}}^2}.$$ Then, $\pi(\eta)=\tfrac{\phi(|\eta|)}{Z_0}$, where recall that $Z_0$ is the normalization constant. Let $I_t:= [\eta^*_t - r, \eta^*_t + r]$, and $|I_t \cap [-1,1]|$ denote the length of the interval $I_t \cap [-1,1]$.

For the first case ($|\uncopt|\le r < 1$), if $\eta^*_t > 0$, then $[\eta^*_t - r, \eta^*_t] \subset I_t \cap [-1,1]$, and otherwise, $[\eta^*_t, \eta^*_t + r] \subset I_t \cap [-1,1]$. Let $s_t := \operatorname{Sign}(\eta^*_t)$, and let $\bar{I}_t := [s_t(|\eta^*_t| -r), s_t|\eta^*_t|]$ denote the interval $[\eta^*_t - r, \eta^*_t]$ when $s_t = 1$, and $[\eta^*_t, \eta^*_t + r]$, otherwise. Then, the following inequalities hold:
\begin{align*} 
\frac{1}{r}\int\limits_{I_t  \cap [-1,1]} \frac{\phi(|\eta|)}{Z_0} ~ d\eta 
&\ge \frac{1}{r}\int\limits_{\bar{I}_t} \frac{\phi(|\eta|)}{Z_0} ~ d\eta \tag{$\bar{I_t} \subset I_t \cap[-1,1]$} \\
&\ge \frac{1}{Z_0} \phi\lrp{\frac{1}{r}  \int\limits_{\bar{I}_t} |\eta| d\eta  } \tag{Convexity of $\phi(\cdot)$ on both sides of $0$} \\
&\overset{(a)}{\ge} \frac{1}{Z_0} ~ \phi(r), 
\end{align*}
where $(a)$ follows by observing that $\phi(\cdot)$ is radially decreasing as well as for all $\eta\in \bar{I}_t$, $|\eta| \le r$. On rearranging the above inequality, we get the desired inequality in this case.

For the second case ($|\eta^*_t| > r$), consider the following inequalities: 
\begin{align*}
    \frac{1}{ |I_t\cap[-1,1]| }~\int\limits_{I_t \cap [-1,1]} \phi(|\eta|) ~ d\eta \overset{(a)}{\ge} \phi\lrp{\frac{1}{ |I_t\cap[-1,1]| }~\int\limits_{I_t \cap [-1,1]} |\eta| d\eta} \overset{(b)}{\ge} \phi(|\eta^*_t|), 
\end{align*}
where in $(a)$, we used Jensen's inequality (the interval $I_t\cap [-1,1]$ lies entirely in either side of $0$ depending on the sign of $\eta^*_t$,  and $\phi(\cdot)$ is a convex function on $[0,1]$ as well as on $[-1,0]$); for $(b)$, observe that $\int_{I_t\cap [-1,1]} |\eta| d\eta \le |\eta^*_t| $, and $\phi$ is radially decreasing. 

Furthermore, since $|I_t\cap [-1,1]| \ge r$, using this, re-arranging the above inequality, and dividing by $Z_0$, we get
\[ \int\limits_{I_t\cap [-1,1]} \frac{\phi(|\eta|)}{Z_0} ~d\eta \ge \frac{r}{Z_0} \phi(|\eta^*_t|) = \frac{r \ln\ln c}{2 |\eta^*_t|\lrp{\ln\frac{c}{|\eta^*_t|}}\lrp{\ln\ln\frac{c}{|\eta^*_t|}}^2} \ge \frac{r \ln\ln c}{2 |\eta^*_t|\lrp{\ln\frac{c}{r}}\lrp{ \ln \ln \frac{c}{r} }^2},\]
where, for the last inequality, we used $r < |\eta^*_t| < 1$. This completes the proof. 
\end{proof}

\begin{proof}[Proof of Lemma~\ref{lem:regretboundary}]
First, observe the following: 
$$f_t(\eta)=f_t(\uncopt)-\frac{\varprocess}{2}(\eta-\uncopt)^2.
$$
For $\rho\in(0,1)$, choose $r:=|\uncopt|-1+\rho$, and set $s_t := \operatorname{Sign}(\eta^*_t) \in \{-1,1\}$. Then, integrating the exponent of the above at the intersection $\{|\eta-\uncopt|\le r\}\cap[-1,1]$, we get
\begin{align*}
    W_t &= \int\limits_{-1}^1 e^{f_t(\eta)}\pi(\eta) d\eta \ge 
    e^{f_t(\uncopt)-\frac{\varprocess r^2}{2}} \int\limits_{\{|\eta-\uncopt|\le r\}\cap[-1,1]} \pi(\eta) d\eta.
\end{align*}
Observe that $\{|\eta-s_t|\le \rho\} \subset \{|\eta-\uncopt|\le r\}\cap[-1,1]$, and on that subset $|\eta|\in[1-\rho,1]$, implying
$$\int\limits_{\{|\eta-\uncopt|\le r\}\cap[-1,1]} \pi(\eta) d\eta\ge \int\limits_{|\eta-s_t|\le \rho}\pi(\eta) d\eta  \overset{(a)}{\ge} \dfrac{2\rho}{Z_0 \lrp{\ln\frac{c}{1-\rho}}\lrp{\ln\ln\frac{c}{1-\rho}}^2},$$
where in $(a)$, we used that $\tfrac{1}{|\eta|} \ge 1$, and that $$ \frac{1}{\lrp{\ln\frac{c}{|\eta|}}\lrp{\ln\ln\frac{c}{|\eta|}}^2} \ge \frac{1}{\lrp{\ln\frac{c}{1-\rho}}\lrp{\ln\ln\frac{c}{1-\rho}}^2}.$$ 
\end{proof}

\begin{proof}[Proof of Proposition~\ref{prop:intermediatebound}]
We first show that~\eqref{eq:Tdelta} holds. Towards this, suppose at time $t > 0$, $|\uncopt|>1$.
Then $|S_t|>\varprocess$. Let  $s_t:=\operatorname{sign}(S_t)\in\{-1,+1\}$. Then, for every $\eta\in[s_t (1-\rho), s_t]$ (to be interpreted as either $[-1, -1 + \rho]$ or $[1-\rho, 1]$, depending on sign of $s_t$), we have $\eta S_t\ge (1-\rho)|S_t|$ and $\eta S_t -\frac{1}{2}\eta^2\varprocess$ is increasing in $\eta$ for $\eta > 0$ and decreasing otherwise. Hence,
\[
f_t(\eta)
 \ge (1-\rho)|S_t| -\frac{\varprocess}{2}(1-\rho)^2
 \ge (1-\rho) \varprocess -\frac{\varprocess}{2}(1-\rho)^2 =\frac{\varprocess}{2}(1-\rho^2).
\]
Therefore,
\[
\ln W_t = \ln \int\limits_{-1}^1 e^{f_t(\eta)} \pi(\eta) d\eta \ge \ln\left(\int_{1-\rho}^{1}\pi(\eta) d\eta\right) + \frac{\varprocess}{2}(1-\rho^2)
.
\]
Now, on $\calE_\alpha$,
since $\sup_t\ln W_t\le \ln(1/\alpha)$, from the above inequality, we must also have
\begin{align*}
\frac{\varprocess}{2}(1-\rho^2) &\le \ln\frac{1}{\alpha}-\ln\left(\int_{1-\rho}^{1}\pi(\eta) d\eta\right)\\
&\overset{(a)}{\le} \ln\frac{1}{\alpha}- \ln\frac{\rho}{Z_0} +  \ln\ln\frac{c}{1-\rho} + 2 \ln\ln\ln\frac{c}{1-\rho},
\end{align*}
which implies, for such a $t$, $V_t \le V_\alpha$, with $V_\alpha$ defined in~\eqref{eq:Tdelta}. In $(a)$, we used that
\[ \ln {~ \int\limits_{1-\rho}^1 \pi(\eta) d\eta ~} \ge \ln\frac{\rho}{Z_0} - \ln\ln\frac{c}{1-\rho} - 2 \ln \ln \ln \frac{c}{1-\rho}, \numberthis\label{eq:lb} \]
which follows by observing that for $\eta \in [1-\rho, 1]$, 
\[\eta < 1, \qquad  \text{ and } \qquad \ln\tfrac{c}{|\eta|}\lrp{\ln\ln\tfrac{c}{\eta}}^2 \le \ln\tfrac{c}{1-\rho}\lrp{\ln\ln\tfrac{c}{1-\rho}}^2,\] 
and hence, $\tfrac{1}{\pi(\eta)} \le Z_0(\ln\tfrac{c}{1-\rho})(\ln\ln\tfrac{c}{1-\rho})^2$.

For~\eqref{eq:regretboundonedelta}, let $\alpha_t:=|\uncopt|-1$. As above,  for $\eta\in[s_t(1-\rho),s_t]$ we have 
$$f_t(\eta)\ge (1-\rho)|S_t|-\frac{\varprocess}{2} =\frac{\varprocess}{2}\big[(1-2\rho)+2(1-\rho)\alpha_t\big].$$
Then,
\begin{align*}
\ln W_t &\ge \ln\left(\int_{1-\rho}^{1}\pi(\eta) d\eta\right) + \frac{\varprocess}{2}\lrp{(1-2\rho)+2(1-\rho)\alpha_t} \\
&\ge \ln\left(\int_{1-\rho}^{1}\pi(\eta) d\eta\right) + \varprocess \lrp{1-\rho} \alpha_t.
\end{align*}
On $\calE_\alpha$, since $\sup_t\ln W_t\le \ln(1/\alpha)$; using this and~\eqref{eq:lb} in the above inequality, and rearranging, 
$$\alpha_t \le \frac{(1+\rho)}{2\varprocess}V_\alpha.$$ 
Further, from the previous part, we also have that $\varprocess\le V_\alpha$, and recall that Lemma~\ref{lem:regretboundary} gives a bound on $R_t$ in this case. Thus, for $t$ such that $|S_t| \ge V_t$, the first term in the bound in Lemma~\ref{lem:regretboundary} is bounded as below:
\begin{align*}
    \frac{\varprocess }{2}(\alpha_t+\rho)^2
    &= \frac{\varprocess \alpha_t^2}{2}+ \frac{ \varprocess \rho^2}{2} + V_t \alpha_t \rho \\
    &\le \frac{ (1+\rho)^2 V^2_\alpha }{8V_t}+ \frac{ V_\alpha \rho^2}{2} + \frac{ \rho (1+\rho) V_\alpha}{2}\\
    &= \frac{(1+\rho)^2 V^2_\alpha}{8 V_t} + \frac{\rho V_\alpha}{2} \lrp{1 + 2\rho}.
\end{align*}
By subtracting the constant prior mass term from Lemma~\ref{lem:regretboundary}, we get the inequality in~\eqref{eq:regretboundonedelta}.
\end{proof}

\section{The Sub-Gaussian Betting Game}\label{app:bettinggame}
In this section, we present the betting game that we consider in this work. To build intuitions, we begin by presenting the game in the simplest setting of Gaussian observations and extend it to the more general sub-Gaussian setting towards the end. 

\paragraph{Gaussian data.} 
Consider a game between two risk-neutral players: a buyer (named Skeptic) and a seller  (named Forecaster) of bets. In this game, Forecaster believes that the distribution of an unseen outcome is $N(0,1)$ (we will generalize this later).  Skeptic is skeptical about  Forecaster's belief and instead believes that it is a mean-shifted Gaussian, i.e., $N(\cdot,1)$, with non-zero mean. 

For every $\eta\in\R$,  Forecaster sells a bet $\eta$ such that every dollar placed on $\eta$ (before observing the data) pays back $e^{\eta X - \eta^2/2}$ dollars, when the data $X$ is revealed. Why are these the only bets offered? Because it is well understood that likelihood ratios are the only admissible bets against a point null~\citep[Chapter 3]{ramdas2024hypothesis}, and $e^{\eta X - \eta^2/2}$ is just the likelihood ratio between $N(\eta,1)$ and $N(0,1)$, which is the log-optimal bet if Skeptic knew the alternative to be $N(\eta,1)$. Note that any such bet is fair from Forecaster's viewpoint, because the expected value of $e^{\eta X - \eta^2/2}$ (the amount Forecaster pays up per dollar of investment in $\eta$) equals one if $X\sim N(0,1)$.

Starting with a unit wealth (i.e., $W_0 = 1$),  Skeptic bets against  Forecaster, hoping to get rich. The betting game between  Forecaster and  Skeptic can be described as below. 

Let $\mathcal P(\R)$ denote the collection of all probability measures with support in $\R$. For $t = 1, 2, \dots$ : 
    \begin{itemize}
    \item Forecaster makes the aforementioned bets available. 
    \item For some $\pi_t \in \mathcal P(\R)$, Skeptic invests a fraction $\pi_t(\eta)$ of the current wealth, i.e.,  $W_{t-1} \pi_t(\eta)$ dollars, in the bet $\eta$, for every $\eta \in \R$. 
    \item Reality reveals the data $X_t$.
    \item Forecaster pays back $e^{\eta X_t - \eta^2/2}$ dollars to Skeptic for every dollar invested in the bet $\eta$. 
    \item The wealth of Skeptic at time $t$ thus becomes $W_t = W_{t-1} \int \pi_t(\eta) e^{\eta X_t - \eta^2/2} d\eta$. 
\end{itemize}

The goal of Skeptic is to choose a betting strategy $(\pi_t)_{t\ge 1}$ whose wealth is close to that of the best fixed betting strategy in hindsight. (A slightly weaker goal for Skeptic is the following: when the true mean is nonzero, they would like to earn as much wealth as an oracle that knows the true mean; of course, when the true mean is zero, neither can earn wealth systematically.)

Formally, the performance of a given betting strategy can be evaluated by considering the difference between its log-wealth and that generated by using the best-in-hindsight strategy. This difference, termed as regret till time $t$, is given by 
\begin{align*} 
R_t &:= \ln \lrp{ \max\limits_{\eta\in \R}~  e^{\eta \sum_{i=1}^t X_i - \eta^2 t / 2}}  -  \ln W_t =\max\limits_{\eta\in \R}~ \lrp{\eta \sum_{i=1}^t X_i - \eta^2 t / 2}  -  \ln W_t. 
\end{align*}
Clearly, $\eta^*_t = S_t/t$ optimizes the wealth in hindsight above, so that the first term just equals $S_t^2/(2t)$. The question we study is whether there is a betting strategy that can make $R_t$ above grow only as $\log\log t$ (since achieving a $\log t$ regret is straightforward, see  Section~\ref{sec:warmuplogTregret}).


In this work, we will use a Bayesian betting strategy that starts with a ``prior'' $\pi \in \mathcal P(\R)$, and at round $t$, bets using a ``posterior'' given by Bayes' rule as:
\[
\pi_t(\eta) = \frac{\pi(\eta) \exp(\eta\sum_{i\leq t-1}  X_i - \frac{\eta^2}{2}(t-1))}{\int_{\R} \exp(\eta\sum_{i\leq t-1}  X_i - \frac{\eta^2}{2}(t-1)) \pi(\eta) d\eta}.
\]
Prior and posterior are put in quotes for two reasons. First, there is nothing Bayesian assumed in the problem setup, and Bayes rule is simply being used as a way of describing the strategy. One could have used ``mixture'' or ``working prior'' instead. Second, the likelihood term is only a valid likelihood in this Gaussian setup, but will no longer be one in the more general sub-Gaussian process setup described later. We claim that in this case, the expression for $W_t$ simplifies to
\[ 
W_t = \int_\R e^{\eta \sum_{i=1}^t X_i - \eta^2 t /2} \pi(\eta) d\eta. 
\]
This is directly checked by noticing that the above definition of $W_t$ results in the normalizing constant $\pi_t(\eta)$ equaling $W_{t-1}$, and thus $W_t$ does indeed satisfy the relationship
$$W_t = W_{t-1} \int \pi_t(\eta) e^{\eta X_t - \eta^2/2} d\eta. $$ 

In our paper, we found it easier to directly work with the expression for the wealth process $W_t$ above than to explicitly specify and manipulate the per-round bets $\pi_t$, which are left implicit.


\begin{remark}
    As mentioned before, the $\eta$-bet that Forecaster makes available to Skeptic in the above game is actually the likelihood ratio between $N(\eta,1)$ and $N(0,1)$, and Skeptic bets on the shifts $\eta$ in the mean. If the Forecaster instead believes that the distribution of the unseen outcome is $N(m_0,1)$, for some $m_0 \ne 0$, and the Skeptic believes that it is a mean-shifted Gaussian, that is, $\{N(m, 1): m\ne m_0\}$, then in the above game, the bets that the Forecaster makes available at each time will be $\{e^{\eta (X-m_0) - \eta^2/2}: \eta\in \R\}$. Again, $e^{\eta(X - m_0) - \eta^2/2}$ is a likelihood ratio between $N(m_0 - \eta, 1)$ and $N(m_0, 1)$. Thus, Forecaster's belief that the distribution is 0-mean Gaussian is without loss of generality, because one can simply redefine $X' = X-m_0$ and proceed. This also holds if the Forecaster changes their prediction to $N(m_t,1)$ in each round $t$: such time-varying forecasts reduce to zero mean forecasts by recentering. Finally, the assumption of the data being unit variance is also without loss of generality. Had the Forecaster claimed $N(0,\sigma)$ instead, we could have redefined the game using $X'=X/\sigma$ (or more generally $X'_t = (X_t-\mu_t)/\sigma_t$). As long as Skeptic is only challenging the mean forecast (which is what we study in this paper), recentering and rescaling do not qualitatively change the game. 
\end{remark}

\paragraph{Sub-Gaussian process.} Recall the functions $S_t$ and $V_t$ introduced in Section~\ref{sec:setup}. We now consider Forecaster who believes that the process $S_t$ is sub-Gaussian with a variance proxy $V_t$, and Skeptic who is skeptical about its mean. The protocol of the game remains the same as earlier,  except that a unit dollar invested in the bet $\eta$ offered by Forecaster at time $t$ now pays back 
\[ e^{\eta (S_t - S_{t-1}) - \frac{\eta^2}{2}(V_t - V_{t-1})}.  \]
With this, the wealth process of Skeptic becomes
\[ W_t = W_{t-1} \int_{\R} \pi_t(\eta) e^{\eta (S_t - S_{t-1}) - \frac{\eta^2}{2} (V_t - V_{t-1}) } d\eta.  \] 
We can similarly define the regret of Skeptic's wealth or strategy. We focus on the performance of the class of strategies whose wealth at time $t$ takes the form
\[ W_t  = \int_{\R} e^{\eta S_t - \frac{\eta^2}{2}V_t} \pi(\eta) d\eta. \]

The sub-Gaussian process (Definition~\ref{def:subgaussian}) captures a large class of distributions. Thus, the game described in the sub-Gaussian process setting is very general. In particular, the game in the Gaussian data  setting considered earlier, can be recovered by setting $S_t = \sum_{i=1}^t X_i$ and $V_t = t$. In fact, the games described above can be generalized even further, as discussed next. 

Forecaster's belief about the nature are captured by the bets he makes available. It can even change across rounds. For example, consider Forecaster who believes that the data is $N(0,1)$ at $t=1$, $N(m_0, \sigma^2)$ at $t=2$ (conditional on $X_1$), and a distribution that is symmetric around $0$ at $t=3$ (conditional on $X_1,X_2$), and so on (in fact, the forecast at time 3  can depend on the first two outcomes: for example, if $X_1$ is positive, $X_3$ can be forecast to be symmetric, and if $X_1$ is negative, $X_3$ can be forecast to any mean-zero and unit variance distribution). Recall from Appendix~\ref{app:subGaussian} that each of these forecasts can be modeled as sub-Gaussian processes. 

More generally, Forecaster does not have to specify their beliefs as sets of distributions, but simply specify the available bets. Since we work with sub-Gaussian bets, they only need to specify how they think $S_t$ and $V_t$ will change at round $t$.

In our game, Skeptic believes that the data is a mean-shifted version of Forecaster's belief.  In the time-varying example considered in the previous paragraph, a unit dollar invested in bet $\eta$ gives back $e^{\eta X_1 - \eta^2/2}$ dollars when the data $X_1$ is revealed at $t=1$. At $t=2$, it returns $e^{\eta(X_2-m_0) - \eta^2 \sigma^2/ 2}$ dollars, at $t=3$, it pays off $e^{\eta X_3 - X^2_3 \eta^2 / 2}$, etc. The game protocol remains the same as earlier. 

Our results deliver low regret in this much more general setting (than the Gaussian setting specified earlier).

\end{document}